\documentclass[twoside,11pt]{article}

\usepackage{jmlr2e}

\usepackage{hyperref}       
\usepackage{url}            
\usepackage{booktabs}       
\usepackage{amsfonts}       
\usepackage{nicefrac}       
\usepackage{microtype}      
\usepackage{color}
\usepackage{amsmath, natbib, graphicx, url}
\usepackage{subfig}
\usepackage{mathtools}
\usepackage{amssymb, tikz}
\usepackage{blkarray,multirow}
\usepackage{pifont}
\usepackage{xfrac}
\usepackage{xr}
\usepackage[normalem]{ulem} 
\usepackage{thm-restate}
\usepackage{cancel,soul,ulem}
\usepackage{mdframed}

\usepackage{algorithm}
\usepackage{algorithmic}

\usepackage{suffix}
\usepackage{decorated}
\usepackage{notation_victor}

\usepackage{etoolbox}

\newtoggle{techreport} 
\toggletrue{techreport}
\newcommand{\onlyTR}[1]{\iftoggle{techreport}{{#1}}{}}
\newcommand{\onlyJournal}[1]{\iftoggle{techreport}{}{#1}}
\newcommand{\TRJournalSplit}[2]{\iftoggle{techreport}{{#1}}{#2}}

\newtheorem{claim}[theorem]{Claim}
\newtheorem{fact}[theorem]{Fact}

\newcommand{\AlgorithmName}[1]{\label{alg:#1}}
\newcommand{\AppendixName}[1]{\label{app:#1}}
\newcommand{\ClaimName}[1]{\label{clm:#1}}

\newcommand{\CorollaryName}[1]{\label{cor:#1}}

\newcommand{\EquationName}[1]{\label{eq:#1}\text{}}

\newcommand{\FactName}[1]{\label{fact:#1}}

\newcommand{\LemmaName}[1]{\label{lem:#1}}

\newcommand{\PropositionName}[1]{\label{prop:#1}}
\newcommand{\SectionName}[1]{\label{sec:#1}}

\newcommand{\TheoremName}[1]{\label{thm:#1}}
    
\newcommand{\Algorithm}[1]{Algorithm~\ref{alg:#1}}
\newcommand{\Appendix}[1]{Appendix~\ref{app:#1}}
\newcommand{\Claim}[1]{Claim~\ref{clm:#1}}

\newcommand{\Corollary}[1]{Corollary~\ref{cor:#1}}

\newcommand{\Equation}[1]{\eqref{eq:#1}}

\newcommand{\Fact}[1]{Fact~\ref{fact:#1}}

\newcommand{\Lemma}[1]{Lemma~\ref{lem:#1}}

\newcommand{\Proposition}[1]{Proposition~\ref{prop:#1}}
\newcommand{\Section}[1]{Section~\ref{sec:#1}}

\newcommand{\Theorem}[1]{Theorem~\ref{thm:#1}}
    
\newenvironment{proofof}[1]{\begin{proof}[of #1]}{\end{proof}}

\newcommand{\Regret}{\operatorname{Regret}}

\newcommand{\bR}{\mathbb{R}}

\newcommand{\cD}{\mathcal{D}}

\newcommand{\cX}{\mathcal{X}}

\newcommand{\cZ}{\mathcal{Z}}

\newcommand{\inner}[2]{\langle\, #1 ,\, #2 \,\rangle}

\renewcommand{\th}{\ifmmode{^{\textrm{th}}}\else{\textsuperscript{th}\ }\fi}
\newcommand{\comment}[1]{}

\definecolor{blu}{rgb}{0,0,1}

\definecolor{gre}{rgb}{0,.5,0}

\definecolor{red}{rgb}{1,0,0}

\def\norm#1{\|#1\|}

\def\R{\mathbb{R}}

\def\E{\mathbb{E}}
\def\X{\mathcal{X}}
\def\D{\mathcal{D}}

\def\KL{\text{KL}}
\def\Reg{\mathrm{Regret}}

\newcommand{\Grad}{\nabla}
\newcommand{\TriBreg}[3]{\ensuremath{D_\Phi({\textstyle{{#1}\atop{#2}}};#3)}}

\usepackage{todonotes}

\jmlrheading{1}{2000}{1-48}{4/00}{10/00}{xxxx}{Huang Fang, Nicholas J.\ A.\ Harvey, Victor S. Portella and Michael P. Friedlander}

\ShortHeadings{OMD and DA: Keeping Pace in the Dynamic Case}{Fang, Harvey, Portella and Friedlander}
\firstpageno{1}

\begin{document}

\title{Online Mirror Descent and Dual Averaging: \\ Keeping Pace in the Dynamic Case \onlyTR{\\\small TECHNICAL REPORT}}

%
\author{\name Huang Fang \email hgfang@cs.ubc.ca \\
       \name Nicholas J.\ A.\ Harvey \email nickhar@cs.ubc.ca \\
       \name Victor S. Portella \email victorsp@cs.ubc.ca \\
       \name Michael P. Friedlander \email mpf@cs.ubc.ca \\
       \addr Department of Computer Science\\
       University of British Columbia\\
       Vancouver, BC V6T 1Z3, Canada}

\editor{}

\maketitle

\begin{abstract}
    Online mirror descent (OMD) and dual averaging (DA)---two fundamental algorithms for online convex optimization---are known to have very similar (and sometimes identical) performance guarantees when used with a \emph{fixed} learning rate.
    Under \emph{dynamic} learning rates, however, OMD is provably inferior to DA and suffers a  linear regret, even in common settings such as prediction with expert advice.
    We modify the OMD algorithm through a simple technique
    that we call \emph{stabilization}.
    We give essentially the same abstract regret bound for OMD with stabilization and for DA by modifying the classical OMD convergence analysis in a careful and modular way that allows for straightforward and flexible proofs.
    Simple corollaries of these bounds show that OMD with stabilization and DA enjoy the same performance guarantees in many applications---even under dynamic learning rates.
    We also shed light on the similarities between OMD and DA and show simple conditions under which stabilized-OMD and
    DA generate the same iterates.
    Finally, we show how to effectively use dual-stabilization with composite cost functions with simple adaptations to both the algorithm and its analysis.
\end{abstract}


\section{Introduction}
\label{sec:intro}

Online convex optimization (OCO) lies in the intersection of machine learning,
convex optimization and game theory. In OCO, a player is required to make a
sequence of online decisions over discrete time steps and each decision incurs a
cost given by a convex function which is only revealed to the player after they
makes that decision. The goal of the player is to minimize what is known as
\emph{regret}: the difference between the total cost and the cost of the best
decision in hindsight. In this setting, algorithms for the player that attain
sublinear regret (usually under mild conditions on the problem) with respect to
the total number of rounds/decisions \(T\) are considered desirable.

Online mirror descent (OMD) and dual averaging (DA) are two important algorithm
templates for OCO from which many classical online learning algorithms can be
described as special cases (see~\citealp{Shalev-Shwartz12} and~\citealp{McMahan17} for some examples). When the total number \(T\) of decisions
to be made is known in advance, OMD and DA achieve exactly the same regret bound when using the same constant learning
rate~\citep{Hazan16}. However, when the number of decisions is \emph{not} known a priori, there
is a fundamental difference in the regret guarantees for OMD and DA with a
similar adaptive learning rate --- while DA can guarantee sublinear regret bound
$O(\sqrt{T})$ for any $T > 0$~\citep{Nesterov09}, there are instances on which
OMD suffers asymptotically linear $\Omega(T)$ regret~\citep{OrabonaP18}.

The aim of this paper is to introduce a \emph{stabilization} technique that bridges the gap between OMD and DA with
dynamic learning rates. We begin by giving almost identical
abstract regret bounds (depending only on the Bregman divergence between
iterates) for stabilized OMD and DA.

We provide simple and clean proofs showing that OMD with stabilization
works similarly to DA in the following aspects:
\begin{itemize}
    \item With the same adaptive learning rate, OMD with stabilization
    achieves exactly the same regret bound as DA. This mirrors the
    situation in which the time horizon is known in advance and both OMD
    and DA use a constant learning rate.
    
    \item For the problem of prediction with expert advice, OMD with stabilization matches the best known regret bound (even with the same constant and with a dynamic learning rate), which was originally achieved by DA (see~\citealp[Section~2.5]{introduction-online-optimization} and~\citealp[Proposition~2.1]{Ger11}).
    \TRJournalSplit{We show that this regret bound is the best that can be achieved with learning rates of the form $c/\sqrt{t}$.}{In our technical report~\citep{TR}, we show that this regret bound is the best that can be achieved with learning rates of the form $c/\sqrt{t}$.}
    
    \item For the problem of prediction with expert advice, we give a concise proof that OMD with stabilization 
    can achieve a first-order regret bound.
    This means that the regret bound does not depend on $T$ but rather on the cost of the best expert up to time $T$, which is no larger.
    Our analysis matches the best known analysis for DA~\citep{Bubeck15,Ger11}, though it is worse than the best known constant in the literature \citep{YaroshinskyES04}.

    \item We formally compare the iterates generated by DA and OMD (with and without stabilization). This sheds light on the reasons why OMD behaves badly with dynamic learning rates. As a corollary of this comparison, we get simple sufficient conditions for the iterates of DA and (dual-)stabilized OMD to match. This mimics the behavior between DA and OMD when the learning rate is fixed.

    \item For composite functions, we show that a proximal variant of dual-stabilized OMD again achieves exactly the same regret bound as regularized DA~\citep{Xiao}. 
\end{itemize}

\section{Related work}

The concept of mirror descent originates from~\cite{MD} and
renewed interest in OMD began with a modern treatment given
by~\cite{BeckT03}. It then received a lot of attention from the
optimization community due to the recent interest in first-order methods
for large-scale problems. For example, see the works
from~\citet{Duchi10compositeobjective, Allen-ZhuO16a} and for further
details and references see the work of~\citet{Beck17a}. DA is due
to~\cite{Nesterov09} and is motivated by the dissatisfying fact that
the convergence of classical subgradient descent methods rely on
decaying step sizes which ultimately give less weight to new
information. This algorithm was later extended to regularized problems
by~\cite{Xiao}. DA is closely related to the
\emph{follow-the-regularized-leader} (FTRL)
algorithm~\citep{Shalev-Shwartz12}. See also the works from
\citet{Bubeck15}, \citet{Hazan16}, and \citet{McMahan17} for complete
descriptions and further references. OMD and DA also grew a lot in
popularity due to applications in online learning
problems~\citep{KakadeST12a,AudibertBL14a}, and the fact that they
generalize a wide range of online learning
algorithms~\citep{Shalev-Shwartz12,McMahan17}. Moreover, OMD and DA have
been used to tackle problems in theoretical computer science such as
graph sparsification~\citep{Allen-ZhuLO15a} and the \(k\)-server
problem~\citep{BubeckCLLM18a}.

Unifying views of online learning algorithms have been shown to be
useful for applications and have drawn recent attention.
\citet{McMahan17} showed how to use adaptive regularization in the
FTRL framework to derive many online learning algorithms, and compared
the different ways OMD and FTRL deal with composite functions. \citet{JoulaniGS17} proposed a unified framework to analyze online
learning algorithms under wildly different assumptions, extending even
to the non-convex case. \citet{Juditsky19} recently
proposed a unified framework called unified mirror descent (UMD) that
encompasses OMD and DA  as special cases. In spite of these unifying
frameworks, the differences between OMD and DA seemed to be overlooked
and one might imagine that the algorithms had similar performance in all settings.

Only recently \cite{OrabonaP18} looked more closely at the difference
between OMD and DA with time-varying learning rates.
They presented counter examples to demonstrate that OMD with a dynamic learning rate could
suffer from linear regret even under well-studied settings such as in the
experts' problem, where the algorithm picks points in the simplex and
the adversary picks linear functions whose gradients have
\(\ell_{\infty}\)-norm at most \(1\). Although this may seem to contradict the well-known \(O(\sqrt{T})\) regret bounds for OMD, it does not.
These sub-linear bounds hold if the algorithm knows the time-horizon from the
start or if the Bregman divergence (with respect to the mirror map) on
the feasible set is bounded. However, the  Kullback-Leibler divergence
is \emph{not} bounded on the simplex. In this paper we explain this
phenomenon and show how the addition of stabilization to OMD fixes this
problem.



For the problem of prediction with expert advice, by means of the
\emph{doubling trick}, \citet{CBFHHSW} show an algorithm with a
sublinear \emph{anytime} regret bound, meaning a bound that holds at
each round of the game. Improved anytime regret bounds were developed by
\citet{AuerCG02}, with a simplified description of the latter given
by~\citet[Section 2.3]{game_prediction}. Sublinear anytime regret bounds
can also be derived from the original work on dual averaging in
\cite{Nesterov09}. Regret bounds that depend on the cost of the best
expert (known as the first-order regret bound) can be traced back to the
work by~\citet{CBFHHSW}. Improved first-order regret bounds were given
by~\citet{AuerCG02} and the current best known first-order regret bound
is from a sophisticated algorithm designed by~\citet{YaroshinskyES04}.

\section{Formal definitions}
\label{sec:preliminaries}

We consider the online convex optimization problem with unknown time horizon.
For each time step $t \in \{1,2,\ldots\}$ the algorithm proposes a point $x_t$
from a closed convex set $\X \subseteq \bR^n$ and an adversary
simultaneously picks a convex cost function $f_t$ which the algorithm has access
to by a first order oracle, that is, for any \(x \in \X\) the algorithm can
compute \(f_t(x)\) and a subgradient \(\ghat \in \subdiff[f_t](x) \coloneqq \{\ghat
\in \bR^n \mid f(z) \geq f(x) + \inner{\ghat}{z - x}~\forall z \in \X \}\). We will assume\footnote{ This holds, for example, if $f$ is finite and convex on an open superset of $\cX$ \cite[Theorem 23.4]{roc70}.} that all cost functions \(f_t\) in this text are subdifferentiable on \(\cX\), that is, meaning that the subdifferential $\partial f(x)$ is non-empty for all $x \in \cX$.

The cost of the iteration at time $t$ is defined as $f_t(x_t)$. In this setting the goal is to produce a
sequence of proposals $\{x_t\}_{t \geq 1}$ that minimizes the \emph{regret} against an unknown comparison point $z \in \X$ that has accrued up
until time $T$:
\begin{equation*}
    \Reg(T, z)
    ~\coloneqq~
    \sum_{t=1}^T f_t(x_t) - \sum_{t=1}^T f_t(z).
\end{equation*}
In this paper we are interested in the case where the algorithm does not
know the time-horizon \(T\) in advance. This implies that our choices
of parameters, including learning rates, cannot depend on \(T\). 


Often our results and proofs will make use of the  \emph{dual
norm} of \(\norm{\cdot}\), defined by
\[
  \|z\|_* = \sup \left\{ \langle z, x \rangle \mid x \in \bR^n, \|x\| \leq 1 \right\} .
\]
Both dual averaging and online mirror descent are parameterized by a
special convex function~\(\Phi\), often referred as a regularizer or a
mirror map (for \(\cD\) and \(\X\)), which among other properties needs\footnote{One
may relax this condition in some cases.
\citet[\S~5.2]{introduction-online-optimization} discusses in depth the
conditions needed on the mirror map.} to be of Legendre
type~\citep[Chapter 26]{roc70}. Formally, throughout the paper we
assume that the function \(\Phi \colon \bar{\cD} \to \bR\) is a closed
convex function such that \(\interior \bar{\cD} \cap \relint \X \neq
\emptyset\) (where \(\relint \X\) denotes the relative interior of
\(\X\)), and whose conjugate is differentiable on \(\bR^n\). Moreover,
we also suppose that \(\Phi\) is of Legendre type, which means that
\(\Phi\) is strictly convex on its domain\footnote{In fact we only need
\(\Phi\) to be strictly convex on some convex subsets of the
domain~\citep[Chapter~26]{roc70}, but for the sake of simplicity we
assume that \(\Phi\) is strictly convex on its entire domain.} and
essentially smooth, that is, for \(\cD \coloneqq \interior \bar{\cD}\)
we have
\begin{itemize}
    \item \(\cD\) is nonempty,
    \item \(\Phi\) is differentiable on \(\cD\), and
    \item \(\lim_{x \to \partial \cD} \norm{\nabla \Phi(x)} =
    +\infty\), where \(\partial \cD\) is the boundary of \(\cD\),
    i.e., \(\partial \cD
    \coloneqq 
    \cl \cD \setminus \cD\).
\end{itemize}

The gradient of the mirror map $\nabla \Phi : \cD \rightarrow \bR^n$
and the gradient of its conjugate $\nabla \Phi^* : \bR^n \rightarrow
\cD$ are mutually inverse bijections between the primal space $\cD$
and the dual space $\bR^n$. We will adopt the following notational
convention. Any vector in the primal space will be written without a
hat, such as $x \in \cD$. The same letter with a hat, namely
$\hat{x}$, will denote the corresponding dual vector:
\begin{equation}
\EquationName{dualvec_notation}
\hat{x} \coloneqq \nabla\Phi(x)
\qquad\text{and}\qquad
x \coloneqq \nabla\Phi^*(\hat{x})
\qquad\text{for all letters $x$}.
\end{equation}

Essential smoothness ensures not only that \(\Phi\) is differentiable on
the interior of its domain, but also that the slope of \(\Phi\)
increases arbitrarily fast near the boundary of its domain. The latter
guarantees, at least intuitively, that the function is increasing
near and in the direction of the boundary of its domain. This property
is fundamental for mirror descent to be well-defined (although not
essential for dual averaging) since it ensures that the Bregman
projection onto \(\X\) is attained by a point on \(\cD\) where \(\Phi\)
is differentiable (we give more details on this on
Section~\ref{sec:comparison}), and uniqueness is a consequence of the
strict convexity of~\(\Phi\). Some mirror maps we shall look at are
classical cases of the OCO literature such as the negative entropy \(x
\in \bR_{+}^n \mapsto \sum_{i = 1}^n x_i \ln x_i\) and the squared
2-norm \(\norm{\cdot}_2^2\), and details on the reasons they are
mirror maps can be found in the works of~\cite{Shalev-Shwartz12,
introduction-online-optimization} and \citet{Bubeck15} (in particular,
\citealp[Section~5.2]{introduction-online-optimization} discusses the
properties of functions of Legendre type and why requiring the conjugate
of the mirror map to be differentiable on the whole space is not
necessary for mirror descent to be well-defined if one restricts the
gradient steps in the dual space in some way).



Given a mirror map $\Phi$, the Bregman divergence with respect to $\Phi$ is defined as
\begin{equation}
\EquationName{BregmanDef}
    D_\Phi(x,y) ~\coloneqq~ \Phi(x) - \Phi(y) - \inner{ \nabla \Phi(y) }{ x - y },
    \qquad \forall x \in \bar{\cD}, \forall y \in \cD.
\end{equation}
Throughout this paper it will be convenient to use the notation
\begin{equation}
\EquationName{TriBreg}
\TriBreg{a}{b}{c}
 ~\coloneqq~ D_\Phi(a,c) - D_\Phi(b,c)
 ~=~ \Phi(a)-\Phi(b) - \inner{ \nabla \Phi(c) }{ a-b }.
\end{equation}
In the important special case where $\Phi(x)=\frac12\|x\|^2_2$, the Bregman divergence relates to the Euclidean distance, i.e., $D_\Phi(x,y)=\frac12\|x-y\|^2_2$. When $\Phi(x)=\sum_{i=1}^n x_i\log x_i$, the Bregman divergence becomes the generalized Kullback-Leibler (KL) divergence.
The projection operator induced by the Bregman divergence is $\Pi^{\Phi}_{\X}$ given by $ \{\Pi^{\Phi}_{\X}\} (y)\coloneqq {\arg \min} \{D_\Phi(x,y) \mid x \in \X \}$ for any \(y \in \X \cap \D\).

A general template for optimization in the mirror descent framework is
shown in Algorithm~\ref{alg:aomd}. The two classical algorithms, online
mirror descent and dual averaging, are incarnations of this, differing
only in how the dual variable $\hat y_t$ is updated.

In this work, for a given initial point $x_1 \in \cX$ of the player, we are
interested in the case when $\sup_{z\in \cX} D_\Phi( z, x_1)$ is bounded, which
still allows $\sup_{z,x \in \cX} D_\Phi( z,x ) $ to be unbounded. In fact, in
the Euclidean setting (i.e., $\Phi = \frac{1}{2}\norm{\cdot}_2^2$), $\sup_{z\in
\cX} D_\Phi( z, x_1 ) $ is bounded if and only if the diameter of $\cX$ given by
$\sup_{x,y \in \cX} \frac{1}{2} \|x-y\|_2^2$ is also bounded. However, for a
general mirror map $\Phi$, assuming $\sup_{z\in \cX} D_\Phi( z, x_1 ) $ is
bounded is strictly weaker than assuming $\sup_{x, y\in \cX} D_\Phi( x, y ) $ is
bounded. This is the case for the well-known experts problem, where $\Phi$
is the negative entropy, $D_{\Phi}$ is the KL-divergence, $\cX$ is the unit
simplex, and $x_1 \coloneqq \frac{1}{n} \vec{1}$, where $\vec{1}$ is the
vector in $\R^n$ with entries all set to 1. In this case, we have $\sup_{z\in
\cX} D_\Phi( z, x_1 ) \leq \ln n $ while $\sup_{z, x\in \cX} D_\Phi( z, x ) =
+\infty$.

\begin{algorithm}[t] \caption{Pseudocode for both online mirror descent and dual
    averaging with adaptive learning rate given by $\eta_t$ on iteration \(t\). These methods differ only in how the iterate
    $\hat{y}_{t+1}$ is updated. 
    }
    \AlgorithmName{aomd}
 \begin{algorithmic}
    \STATE {\bfseries Input:} ${x}_1 \in \X \cap \D, \eta:\mathbb{N} \rightarrow
    \R_{>0} $.
    \FOR{$t=1,2,\ldots$}
         \STATE Incur cost $f_t(x_t)$ and receive $\ghat_t \in \partial f_t(x_t)$
         \STATE $\hat{x}_t = \nabla \Phi(x_t)$
         \STATE [\textbf{OMD update}] $\hat{y}_{t+1} = \hat{x}_t - \eta_t \ghat_t$
         \STATE [\textbf{DA update}]~~~ $\hat{y}_{t+1} = \hat{x}_1 - \eta_t \sum_{i \leq t} \ghat_i$
         \STATE ${y}_{t+1} = \nabla \Phi^*( \hat{y}_{t+1} )$
         \STATE $x_{t+1} = \Pi^{\Phi}_{\X} (y_{t+1}) $
    \ENDFOR
 \end{algorithmic}
 \end{algorithm}

 To conclude, throughout the paper we shall try to stick to the
 following naming convention: greek letters will denote scalars,
 lower-case roman letters will denote vectors, and capital caligraphic
 letters shall denote sets. Any deviations from this convention (or the hat notation as in~\eqref{eq:dualvec_notation}) are intended to follow other conventions in the literature. 
 
\section{The relationship between OMD and DA}

In this section, we present a detailed review of some known properties of OMD and DA. Our goal is to summarize known similarities and differences between the guarantees on the regret for these algorithms in the fixed-time and anytime settings.

\subsection{OMD and DA with constant learning rate}

When the time horizon $T$ is known in advance, a constant learning rate that depends on $T$ can be adopted in many algorithms for OCO to achieve sublinear regret. In particular, OMD and DA with the same fixed learning rate enjoy exactly the same regret bound.
\begin{theorem}[\protect{\citealp[Thm.~1]{Nesterov09}, \citealp[Thm.~5.6]{Hazan16}}]
\label{thm:bound_constant_rate}
Suppose that $\Phi$ is $\rho$-strongly convex with respect to a norm $\| \cdot
\|$ and pick a constant learning rate $\eta_t \coloneqq \eta > 0$ for all \(t
\geq 1\). Let $\{x_t\}_{t \geq 1}$ be the sequence of iterates generated by
Algorithm~\ref{alg:aomd}. Then for any sequence of convex functions $\{f_t\}_{t
\geq 1}$ with $f_t \colon \X \rightarrow \R$ for each \(t \geq 1\), the following bound holds for both
OMD and DA updates,
\begin{equation}
\label{eq:ClassicalRegretBound}
  \mathrm{Regret}(T, z) ~\leq~ \sum_{t=1}^T \frac{\eta
    \|\ghat_t\|_*^2}{2\rho} + \frac{ D_\Phi(z, x_1) }{\eta},
\end{equation}
for any comparison point $z \in \cX$.
\end{theorem}
Interestingly, though OMD and DA with constant learning rate have similar regret
bounds, the proofs used to derive these bound tend to be quite different.

\subsection{OMD and DA with dynamic learning rate}

In the unknown time horizon scenario, a dynamic learning rate with $\eta_t
\propto 1/\sqrt{t}$ is usually adopted in the literature of online
learning~\citep{BeckT03,Zinkevich03}. Moreover, when the Bregman divergence (with respect to $\Phi$) on the domain $\X$ is bounded,
both OMD and DA with learning rate $\eta_t \propto 1/\sqrt{t}$ can achieve
$O(\sqrt{T})$ regret bounds (with differing constants). However, when the Bregman divergence on
$\X$ is unbounded, OMD is provably worse than DA as the next theorem shows.

\begin{theorem}[{\protect Linear regret for OMD, \citealp[Thm.~3]{OrabonaP18}}]
\label{thm:counter_ogd}
Set $\eta_t \coloneqq 1/\sqrt{t}$ for each \(t \geq 1\). Let $\{x_t\}_{t \geq 1}$ denote the sequence of
iterates generated by Algorithm~\ref{alg:aomd} with OMD update.
There exists a sequence of convex $1$-Lipschitz continuous functions $\{f_t\}_{t=1}^T$ and an initial point $x_1\in\X$ such that
\[
    D_\Phi(z,x_1)~~\text{is bounded}
    \quad\text{and}\quad
    \mathrm{Regret}(T, z) \:=\: \Omega( T ).
\]
In contrast, Algorithm~\ref{alg:aomd} with the DA update can always guarantee sublinear
regret bound $O(\sqrt{T})$ using a similar learning rates (which differ only by
constants).
\end{theorem}

Moreover, folklore examples show that for offline 1-dimensional gradient descent (i.e., mirror descent with Euclidean regularization), 
a learning rate of either $o(1/\sqrt{t})$ or $\omega(1/\sqrt{t})$ cannot
achieve regret $O(\sqrt{t})$ for all $t > 0$.
Therefore OMD with learning rates
of the form $t^{-\alpha}$ with $\alpha > 0$ may not have optimal regret
guarantees when the Bregman divergence on \(\X\) is unbounded. A natural question is if we can improve OMD
to make it provably work with dynamic learning rates. In the next section we
provide a fix for adaptive OMD through stabilization and later we show its
connection to adaptive DA.

 \section{Stabilized OMD}
 \SectionName{StabilizedOMD}

 As shown in Theorem~\ref{thm:counter_ogd}, \citet{OrabonaP18} proved that OMD with the standard dynamic learning
 rate ($\eta_t \propto 1/\sqrt{t}$) can incur regret linear in $T$ when the feasible set $\X$ is unbounded Bregman divergence, that is, $\sup_{x,z \in \X} D_\Phi(z,x) = \infty$. We introduce a stabilization technique that resolves
 this problem, allowing OMD to support a dynamic learning rate and
 perform similarly to DA even when the Bregman divergence on \(\X\) is unbounded.
 
 The intuition for the idea is as follows. Suppose \(\cZ \subseteq \X\)
 is a set of comparison points with respect to which we wish our
 algorithm to have low regret. Usually, we assume $\sup_{z \in \cZ}
 D_\Phi(z,x_1)$ is bounded, that is, the initial point is not too far
 (with respect to the Bregman divergence) from any comparison point.
 Since $\sup_{z
 \in
 \cZ} D_\Phi(z,x_1)$ is bounded (but not necessarily $\sup_{z \in \cZ,x \in \cX}
 D_\Phi(z,x)$), the point $x_1$ is the only point in $\X$ that is known
 to be somewhat close (w.r.t.\ the Bregman divergence) to all the other
 points in \(\X\). Thus, iterates computed by the algorithm should remain
 reasonably close to $x_1$ so that no other point \(z \in \cZ\) is too
 far from the iterates. If there were such a point \(z\), an adversary
 could later chose functions so that picking \(z\) in every round would
 incur low loss. At the same time, OMD would take many iterations to
 converge to \(z\) since consecutive OMD iterates tend to be close
 w.r.t.\ the Bregman divergence. That is, the algorithm would have high
 regret against~\(z\). To prevent this, the stabilization technique
 modifies each iterate $x_t$ to mix in a small fraction of $x_1$. This
 idea is not entirely new: it appears, for example, in the original Exp3
 algorithm~\citep{ACBFS02}, although for different reasons.
 
 There are two ways to realize the stabilization idea.
 
 \begin{itemize}
    \item \textbf{Primal Stabilization.} 
    Replace $x_t$ with a convex combination of $x_t$ and $x_1$.
    \item \textbf{Dual Stabilization.} 
    Replace $\hat{y}_t$ with a convex combination of $\hat{y}_t$ and $\hat{x}_1$
    (Recall from Algorithm~\ref{alg:aomd} that $\hat{y}_t$ is the dual iterate computed by taking a gradient step). An illustration for dual stabilization is shown in Figure~\ref{fig:illustration}.
 \end{itemize}

 \begin{figure}
 \centering
     \includegraphics[width=0.65\columnwidth]{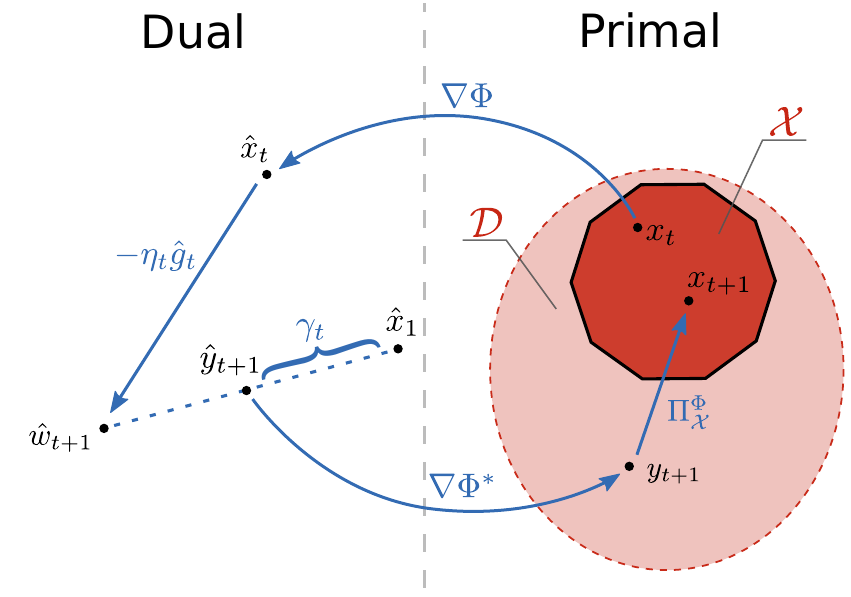} 
     \caption{Illustration of the \(t\)-th iteration of DS-OMD. }
     \label{fig:illustration}
 \end{figure}

 After early versions of the paper were available, we were informed that ideas similar to primal stabilization had appeared in the Robust Optimistic Mirror Descent algorithm~\citep{KangarshahiHSC18} and in the Twisted Mirror Descent(TMD) algorithm~\citep{GyorgyS16a}.
 The setting of the former is different since they perform optimistic steps and their results are somewhat weaker in terms of constant factors and since they cannot handle Bregman projections.  In the latter case, TMD is a meta-algorithm that has primal stabilization as a special case. But the authors only use primal stabilization for the experts' problem~\citep[Example~6]{GyorgyS16a}. In our work we extend the idea of primal stabilization for cases beyond the setting of prediction with expert advice.

 \subsection{Dual-stabilized OMD}
 \SectionName{generic-ds-omd}
 
 Algorithm~\ref{alg:somd} gives pseudocode showing our modification of OMD to incorporate dual stabilization.
 \Theorem{somd} analyzes it without assuming strong convexity of $\Phi$.
 
 \begin{algorithm}[t]
    \caption{Dual-stabilized OMD (DS-OMD). The parameters $\gamma_t$ control the
    amount of stabilization.}
    \label{alg:somd}
    \begin{algorithmic}
      \STATE {\bfseries Input:}
      ${x}_1 \in \X \cap \cD,
      \eta:\mathbb{N} \rightarrow \R_+,\
      \gamma : \mathbb{N} \to (0, 1] $
      \FOR{$t=1,2,\ldots$}
         \STATE Incur cost $f_t(x_t)$ and receive $\ghat_t \in \partial f_t(x_t)$
         \STATE \vspace{-21pt} \addtolength{\jot}{-3pt}
             \begin{flalign}
             & \hat{x}_t = \nabla \Phi(x_t)&&
              \text{$\rhd$  map primal iterate to dual space}
              \hspace{3cm}
                \nonumber
             \\
             & \hat{w}_{t+1} = \hat{x}_t - \eta_t \ghat_t
             &&
              \text{$\rhd$  gradient step in dual space}
             \EquationName{GradientStep} \\
             & \hat{y}_{t+1} = \gamma_t \hat{w}_{t+1} + (1-\gamma_t) \hat{x}_1
             &&
             \text{$\rhd$ stabilization in dual space} 
             \EquationName{DSOMDHyp}\\
             & y_{t+1} = \nabla \Phi^*( \hat{y}_{t+1} )
             &&
             \text{$\rhd$ map dual iterate to primal space} 
             \nonumber\\
             & x_{t+1} = \Pi^{\Phi}_{\X}(y_{t+1}) \EquationName{DSOMDHyp2}
             &&
             \text{$\rhd$  project onto feasible region} 
             \end{flalign}
             \vspace{-21pt}
    \ENDFOR
 \end{algorithmic}
 \end{algorithm}
 
 \begin{theorem}[Regret bound for dual-stabilized OMD] \TheoremName{somd}
Assume we have $\eta_t \geq \eta_{t+1} > 0$ for each~$t > 1$. Define $\gamma_t = \eta_{t+1}/\eta_t \in
 (0,1]$ for all $t \geq 1$.  Let $\{f_t\}_{t \geq 1}$ be a sequence of
 convex functions with $f_t \colon \X \rightarrow \R$ for each \(t \geq
 1\). Let $\{x_t\}_{t \geq 1}$ and $\{\what_t\}_{t \geq 2}$ be as in
 Algorithm~\ref{alg:somd}. Then, for all \(T > 0\) and \(z \in \cX\),
 \begin{equation}
 \EquationName{RegretBound}
   \Reg(T,z) ~\leq~ \sum_{t=1}^T \frac{\TriBreg{x_t}{x_{t+1}}{w_{t+1}}}{\eta_t} +  \frac{ D_\Phi(z, x_1) }{\eta_{T+1}}.
 \end{equation}
 \end{theorem}
 
 Note that strong convexity of $\Phi$ is \emph{not} assumed. As we will
 see in \Section{strongly-convex-ds-omd}, the term
 $\TriBreg{x_t}{x_{t+1}}{w_{t+1}}$ can be easily bounded by the dual norm of \(\ghat_t\) when the mirror
 map is strongly convex. This yields sublinear regret for $\eta_t
 \propto 1/\sqrt{t}$, which is not the case for OMD when \(\sup_{z \in
 \cZ,x \in \X} D_{\Phi}(z,x) = +\infty\), where \(\cZ \subseteq \cX\) is
 a fixed set of comparison points.
 
 \comment{An advantage of dual-stabilized OMD over dual averaging is, in
 our opinion, the transparent nature of our analysis. Our proof of
 Theorem~\ref{thm:somd} is very different from Nesterov's \citep[Theorem
 1]{Nesterov09} and requires only a few modifications to the classical
 analysis of OMD.}

 \begin{proofof}{\Theorem{somd}}\mbox{}

\begin{mdframed}
Let $z \in \X$. The first step is the same as in the standard OMD proof.
\begin{align}
    f_t(x_t) - f_t(z) &~\leq~ \inner{ \ghat_t }{ x_t - z } \nonumber 
    \qquad
    &&\text{(subgradient ineq.)}
    \\
    &~=~ \frac{1}{\eta_t} \inner{ \hat{x}_t - \hat{w}_{t+1}}{ x_t - z }
    \qquad
    &&\text{(by \eqref{eq:GradientStep})}
    \nonumber \\
    &~=~ \frac{1}{\eta_t} \big(D_\Phi(x_t,w_{t+1}) - D_\Phi(z,w_{t+1}) + D_\Phi(z,x_t) \big)
    \quad
    &&\text{(by Prop.~\ref{prop:BregmanTriangle})}
    .  \EquationName{regret_at_each_round}
\end{align}
\end{mdframed}

The next step exhibits the main point of stabilization. Without stabilization we
would have $x_{t+1} = \Pi^\Phi_\cX(w_{t+1})$ and \(D_{\Phi}(z, w_{t+1}) \geq
D_{\Phi}(z, x_{t+1}) + D_{\Phi}(x_{t+1}, w_{t+1})\) by
\Proposition{BregmanProjection}, so \eqref{eq:regret_at_each_round} would lead
to a telescoping sum involving $D_\Phi(z,\cdot)$ \emph{if the learning rate were
fixed}. With a dynamic learning rate the analysis is trickier: we need a claim
that leads to telescoping terms by relating $D_\Phi(z, w_{t+1} )$ to
$D_\Phi(z,x_{t+1})$.

\begin{claim}
\ClaimName{NewDSOMD}
Assume that $\gamma_t = \eta_{t+1}/\eta_t \in (0,1]$.
Then
\[
\eqref{eq:regret_at_each_round}
~\leq~
\frac{\TriBreg{x_t}{x_{t+1}}{w_{t+1}}}{\eta_t} + \underbrace{\Big( \frac{1}{\eta_{t+1}} - \frac{1}{\eta_t} \Big)}_{\text{telescopes}} D_\Phi(z, x_1) +
\underbrace{\frac{D_\Phi(z, x_t)}{\eta_t} -\frac{D_\Phi(z, x_{t+1})}{\eta_{t+1}}}_{\text{telescopes}}.
\]
\end{claim}
\begin{proof}
First we derive the inequality
\begin{alignat*}{3}
&           \gamma_t  \big( D_\Phi(z,w_{t+1}) - D_\Phi(x_{t+1},w_{t+1}) \big)
    \:+\: (1-\gamma_t) &&D_\Phi(z,x_1) \\
~\geq~&     \gamma_t  \TriBreg{z}{x_{t+1}}{w_{t+1}}
    \:+\: (1-\gamma_t) \TriBreg{z}{x_{t+1}}{x_1}
    &&\text{(since $D_\Phi(x_{t+1},x_1) \geq 0$ and $\gamma_t \leq 1$)}\\
~=~&        \TriBreg{z}{x_{t+1}}{y_{t+1}}
    &&\text{(by \Proposition{BregmanIdentity} and \eqref{eq:DSOMDHyp})}\\
~\geq~&     D_\Phi(z,x_{t+1})
    &&\text{(by \Proposition{BregmanProjection} and \eqref{eq:DSOMDHyp2})}.
\end{alignat*}
Rearranging and using $\gamma_t>0$ yields
\begin{equation}\EquationName{DSOMDRearranged}
D_\Phi(z, w_{t+1})
    ~\geq~ 
        D_\Phi(x_{t+1}, w_{t+1}) 
        - \Big(\frac{1}{\gamma_t}-1\Big) D_\Phi(z, x_1)
        + \frac{1}{\gamma_t} D_\Phi( z, x_{t+1} ).
\end{equation}
Plugging this into \eqref{eq:regret_at_each_round} yields
\begin{align*}
\eqref{eq:regret_at_each_round}
    &~=~ \frac{1}{\eta_t} \Big(
            D_\Phi(x_t,w_{t+1}) - D_\Phi(z,w_{t+1}) + D_\Phi(z,x_t)
        \Big) \\
    &
    \begin{aligned}
        ~\leq~ \frac{1}{\eta_t} \Bigg(
            &D_\Phi(x_t,w_{t+1}) 
            - D_\Phi(x_{t+1}, w_{t+1}) 
            ~+ 
            \\
            &\Big(\frac{1}{\gamma_t}-1\Big) D_\Phi(z, x_1)
            - \frac{1}{\gamma_t} D_\Phi( z, x_{t+1} )
            + D_\Phi(z,x_t)
        \Bigg).
    \end{aligned}
    \qquad \text{(by \eqref{eq:DSOMDRearranged})}
\end{align*}
The claim follows by the definition of $\gamma_t$.
\end{proof}

\begin{mdframed}
The final step is very similar to the standard OMD proof. 
Summing \eqref{eq:regret_at_each_round} over $t$
and using \Claim{NewDSOMD} leads to the desired telescoping sum.
\begin{align*}
&\sum_{t=1}^T \big( f_t(x_t) - f_t(z) \big)\\
    ~\leq~& \sum_{t=1}^T \Bigg(\frac{
    \TriBreg{x_t}{x_{t+1}}{w_{t+1}}}{\eta_t} + \Big( \frac{1}{\eta_{t+1}} - \frac{1}{\eta_t} \Big) D_\Phi(z, x_1) + \frac{D_\Phi(z, x_t)}{\eta_t} -\frac{D_\Phi(z, x_{t+1})}{\eta_{t+1}} \Bigg) \nonumber
    \\
    ~\leq~& \sum_{t=1}^T \frac{\TriBreg{x_t}{x_{t+1}}{w_{t+1}}}{\eta_t}
      \:+\: \left( \frac{1}{\eta_1} + \sum_{t=1}^T \left( \frac{1}{\eta_{t+1}} - \frac{1}{\eta_t} \right)  \right){D_\Phi(z, x_1)}  \nonumber
    \\
    ~=~& \sum_{t=1}^T \frac{\TriBreg{x_t}{x_{t+1}}{w_{t+1}}}{\eta_t} \:+\: \frac{ D_\Phi(z, x_1) }{\eta_{T+1}}.
\end{align*}
\end{mdframed}
\end{proofof}

\subsection{Primal-stabilized OMD}

\Algorithm{ps_omd} gives pseudocode showing our modification of OMD to incorporate primal stabilization. Interestingly, as we shall see in the next theorem, we need convexity of \(D_{\Phi}(z,\cdot)\) for \(z \in \X\) to get interesting bounds on the regret when using primal stabilization.

\begin{algorithm}[H]
   \caption{Online mirror descent with primal stabilization.}
   \AlgorithmName{ps_omd}
\begin{algorithmic}
   \STATE {\bfseries Input:} $x_1 \in \cX \cap \cD ,\, 
   \eta:\mathbb{N} \rightarrow \R ,\,
   \gamma : \mathbb{N} \rightarrow \R $.
   \FOR{$t=1,2,\ldots$}
        \STATE Incur cost $f_t(x_t)$ and receive $\ghat_t \!\in\! \partial f_t(x_t)$
        \STATE \vspace{-21pt} \addtolength{\jot}{-2pt}
            \begin{flalign}
            & \hat{x}_t = \nabla \Phi(x_t) &&\text{$\rhd$ map primal iterate to dual space} \nonumber\\
            & \hat{w}_{t+1} = \hat{x}_t - \eta_t \ghat_t &&\text{$\rhd$ gradient step in dual space} \EquationName{GradientStep2}\\
            & w_{t+1} = \nabla \Phi^*(\hat{w}_{t+1}) &&\text{$\rhd$ map dual iterate to primal space}
            \EquationName{PSOMDHyp3}\\
            & y_{t+1} = \Pi^{\Phi}_{\X} (w_{t+1}) &&\text{$\rhd$ project onto feasible region}\hspace{3cm} \EquationName{PSOMDHyp2}\\
            & {x}_{t+1} = \gamma_t {y}_{t+1} + (1-\gamma_t) x_1 &&\text{$\rhd$ stabilization in primal space} \EquationName{PSOMDHyp}
            \end{flalign}
            \vspace{-21pt}
   \ENDFOR
\end{algorithmic}
\end{algorithm}

\begin{theorem}[Regret bound for primal-stabilized OMD]
\label{thm:psomd}
Assume $\eta_t \geq \eta_{t+1} > 0$ for each $t > 1$.
Define $\gamma_t = \eta_{t+1}/\eta_t \in (0,1]$ for all $t \geq 1$.
Let $\{x_t\}_{t \geq 1}$ be the sequence of iterates generated by \Algorithm{ps_omd}.
Furthermore, assume that
\begin{equation}
    \EquationName{PSOMDConvexity}
\text{for all $z \in \cX$, the map $x \mapsto D_\Phi(z,x)$ is convex on $\cX$.}
\end{equation}
Then for any sequence of convex
functions $\{f_t\}_{t \geq 1}$ with each $f_t : \X \rightarrow \R$,
\begin{equation}
\EquationName{PSRegretBound}
  \mathrm{Regret}(T, z) ~\leq~ \sum_{t=1}^T \frac{\TriBreg{x_t}{y_{t+1}}{w_{t+1}}}{\eta_t} + \frac{ D_\Phi(z, x_1) }{\eta_{T+1}}
  \quad \forall T > 0.
\end{equation}
\end{theorem}

 \begin{proofof}{\Theorem{psomd}}\mbox{}

Let \(z \in \X\). The first step is identical to the proof of \Theorem{somd} since the
update rule in \eqref{eq:GradientStep2} is exactly the same as
\eqref{eq:GradientStep}. Therefore, we have that
\eqref{eq:regret_at_each_round} holds, that is, 
\begin{align*}
    f_t(x_t) - f_t(z) ~\leq~ \frac{1}{\eta_t} \big(D_\Phi(x_t,w_{t+1}) - D_\Phi(z,w_{t+1}) + D_\Phi(z,x_t) \big).
\end{align*}

Now instead of \Claim{NewDSOMD} we use the next claim, which is similar to 
 \Claim{NewDSOMD} but replaces
$\TriBreg{x_t}{x_{t+1}}{w_{t+1}}$ with
$\TriBreg{x_t}{y_{t+1}}{w_{t+1}}$.

\begin{claim}
\ClaimName{NewPSOMD}
Assume that $\gamma_t = \eta_{t+1}/\eta_t \in (0,1]$.
Then
\[
\eqref{eq:regret_at_each_round}
~\leq~
\frac{\TriBreg{x_t}{y_{t+1}}{w_{t+1}}}{\eta_t} + \underbrace{\Big( \frac{1}{\eta_{t+1}} - \frac{1}{\eta_t} \Big)}_{\text{telescopes}}
D_\Phi(z, x_1) +
\underbrace{\frac{D_\Phi(z, x_t)}{\eta_t} -\frac{D_\Phi(z, x_{t+1})}{\eta_{t+1}}}_{\text{telescopes}}.
\]
\end{claim}
\begin{proof}
First, we derive the inequality
\begin{align*}
        &  \gamma_t   \big( D_\Phi(z,w_{t+1}) - D_\Phi(y_{t+1},w_{t+1}) \big)
   \:+\: (1-\gamma_t) D_\Phi(z,x_1) \\
~=~&    \gamma_t   \TriBreg{z}{y_{t+1}}{w_{t+1}}
   \:+\: (1-\gamma_t) D_\Phi(z,x_1) \\
~\geq~&     
    \gamma_t D_\Phi(z,y_{t+1}) + (1-\gamma_t) D_\Phi(z,x_1)
    &&\qquad\text{(by Prop.~\ref{prop:BregmanProjection} and \eqref{eq:PSOMDHyp2})}\\
~\geq~&     D_\Phi(z,x_{t+1})
    &&\qquad\text{(by \eqref{eq:PSOMDHyp} and \eqref{eq:PSOMDConvexity})}.
\end{align*}
Rearranging and using $\gamma_t>0$ yields
\begin{equation}\EquationName{PSOMDRearranged}
D_\Phi(z, w_{t+1})
    ~\geq~ 
        D_\Phi(y_{t+1}, w_{t+1}) 
        - \Big(\frac{1}{\gamma_t}-1\Big) D_\Phi(z, x_1)
        + \frac{1}{\gamma_t} D_\Phi( z, x_{t+1} ).
\end{equation}
Plugging this into \eqref{eq:regret_at_each_round} yields
\begin{align*}
\eqref{eq:regret_at_each_round}
    &~=~ \frac{1}{\eta_t} \Big(
            D_\Phi(x_t,w_{t+1}) - D_\Phi(z,w_{t+1}) + D_\Phi(z,x_t)
        \Big) \\
    &
    \begin{aligned}
        ~\leq~    \frac{1}{\eta_t} \Bigg(
            &D_\Phi(x_t,w_{t+1}) 
            - D_\Phi(y_{t+1}, w_{t+1}) 
            ~+ \\
            &~ \Big(\frac{1}{\gamma_t}-1\Big) D_\Phi(z, x_1) - \frac{1}{\gamma_t} D_\Phi( z, x_{t+1} )
            + D_\Phi(z,x_t)
        \Bigg).
    \end{aligned}
    \qquad \text{(by \eqref{eq:PSOMDRearranged})}
\end{align*}
The claim follows by the definition of $\gamma_t$.
\end{proof}

\begin{mdframed}
The final step is very similar to the proof of \Theorem{somd}. The only difference is that we are using \Claim{NewPSOMD} instead of \Claim{NewDSOMD} and we replace $\TriBreg{x_t}{x_{t+1}}{w_{t+1}}$ with $\TriBreg{x_t}{y_{t+1}}{w_{t+1}}$. Formally, summing \eqref{eq:regret_at_each_round} over $t$ and using \Claim{NewPSOMD} leads to the desired telescoping sum, that is,
\begin{align*}
&\sum_{t=1}^T \big( f_t(x_t) - f_t(z) \big)\\
    ~\leq~& \sum_{t=1}^T \Bigg(\frac{
    \TriBreg{x_t}{y_{t+1}}{w_{t+1}}}{\eta_t} + \Big( \frac{1}{\eta_{t+1}} - \frac{1}{\eta_t} \Big) D_\Phi(z, x_1) + \frac{D_\Phi(z, x_t)}{\eta_t} -\frac{D_\Phi(z, x_{t+1})}{\eta_{t+1}} \Bigg) \nonumber
    \\
    ~\leq~& \sum_{t=1}^T \frac{\TriBreg{x_t}{y_{t+1}}{w_{t+1}}}{\eta_t}
      \:+\: \left( \frac{1}{\eta_1} + \sum_{t=1}^T \left( \frac{1}{\eta_{t+1}} - \frac{1}{\eta_t} \right)  \right){D_\Phi(z, x_1)}  \nonumber
    \\
    ~=~& \sum_{t=1}^T \frac{\TriBreg{x_t}{y_{t+1}}{w_{t+1}}}{\eta_t} \:+\: \frac{ D_\Phi(z, x_1) }{\eta_{T+1}}.
\end{align*}
\end{mdframed}

\end{proofof}

\subsection{Dual averaging}

In this section, we show that Nesterov's dual averaging algorithm can be obtained from a small modification to dual-stabilized online mirror descent.
Furthermore our proof of \Theorem{somd} can be adapted to analyze DA.

The main difference between DS-OMD and dual averaging is in the gradient step, as we now explain.
In iteration $t$ of DS-OMD, the gradient step in \eqref{eq:GradientStep} is taken from $\hat{x}_t$, the dual counterpart of the iterate $x_t$:
\begin{align}\nonumber
\text{DS-OMD gradient step:}\quad
\hat{w}_{t+1} &~=~ \hat{x}_t - \eta_t \ghat_t.
\intertext{
Suppose that the algorithm is modified so that the gradient step is taken from $\hat{y}_t$, the dual point from iteration $t$ \emph{before} projection onto the feasible region. (Here $\hat{y}_1$ is defined to be $\hat{x}_1$.) The resulting gradient step is: 
}
\text{Lazy gradient step:}\quad
\hat{w}_{t+1} &~=~ \hat{y}_t - \eta_t \ghat_t.
\EquationName{Lazy}
\end{align}
As before, we set
\begin{equation}
\EquationName{DAStabilization}
\hat{y}_{t+1} ~=~ \gamma_t \hat{w}_{t+1} + (1-\gamma_t) \hat{x}_1
\end{equation}
where $\gamma_t = \eta_{t+1}/\eta_{t}$.
Then a simple inductive proof yields the following claim.

\begin{claim}
\ClaimName{DAIterative}
$\hat{w}_t = \hat{x}_1 - \eta_{t-1} \sum_{i<t} \ghat_i$ and $\hat{y}_t = \hat{x}_1 - \eta_t \sum_{i<t} \ghat_i$
for all $t>1$.
\end{claim}

Thus, the algorithm with the lazy gradient step can be written as in \Algorithm{DualAveraging}.
This is equivalent to \Algorithm{aomd} with the DA update, except that
$\eta_t$ in \Algorithm{aomd}
corresponds to
$\eta_{t+1}$ in \Algorithm{DualAveraging}. In the next theorem we analyze DA with similar techniques to the ones used in Theorems~\ref{thm:somd} and~\ref{thm:psomd}.

\begin{algorithm}[H]
   \caption{Dual averaging with learning rate re-indexed as $\eta_2, \eta_3, \ldots$}
   \AlgorithmName{DualAveraging}
   \begin{algorithmic}
     \STATE {\bfseries Input:}
     ${x}_1 \in \X \cap \cD,\ 
     \eta:\mathbb{N} \rightarrow \R_+,\
     \gamma : \mathbb{N} \to (0, 1] $
    \STATE $\hat{y}_1 = \nabla \Phi(x_1)$
     \FOR{$t=1,2,\ldots$}
        \STATE Incur cost $f_t(x_t)$ and receive $\ghat_t \in \partial f_t(x_t)$
        \STATE \makebox[2in][l]{$\hat{y}_{t+1} = \hat{x}_1 - \eta_{t+1} \sum_{i \leq t} \ghat_i$}
               $\rhd$ dual averaging update
        \STATE \makebox[2in][l]{${y}_{t+1} = \nabla \Phi^*( \hat{y}_{t+1} )$}
                $\rhd$ map dual iterate to primal space
        \STATE \makebox[2in][l]{$x_{t+1} =
        \Pi^{\Phi}_{\X}(y_{t+1}) $}
                $\rhd$ project onto feasible region
   \ENDFOR
\end{algorithmic}
\end{algorithm}


\comment{
    \vspace{2cm}
    If we reindex so
    that $\eta_t$ becomes $\eta_{t-1}$, then the resulting algorithm
    is equivalent to dual averaging (as formulated in
    Algorithm~\ref{alg:aomd}).  Thus, dual averaging may also be
    viewed as related to stabilization.
}

\begin{theorem}[Regret bound for dual averaging]
\TheoremName{DualAveraging}
Assume we have $\eta_t \geq \eta_{t+1} > 0$ for each~$t > 1$.
Let $\{x_t\}_{t \geq 1}$ be the sequence of iterates generated by \Algorithm{DualAveraging}.
Then for any sequence of convex
functions $\{f_t\}_{t \geq 1}$ with each $f_t : \X \rightarrow \R$,
\begin{equation}
\EquationName{DARegretBound}
  \mathrm{Regret}(T, z) ~\leq~ \sum_{t=1}^T 
  \frac{\TriBreg{x_t}{x_{t+1}}{\Grad\Phi^*(\hat{x}_t - \eta_t \ghat_t)}}{\eta_t}
  + \frac{ D_\Phi(z, x_1) }{\eta_{T+1}}
  \quad \forall T > 0.
\end{equation}
\end{theorem}

\begin{proofof}{\Theorem{DualAveraging}} \mbox{}

\begin{mdframed}
Let $z \in \X$. The first step is very similar to the proof of \Theorem{somd}.
\begin{align}
    f_t(x_t) - f_t(z) &~\leq~ \inner{ \ghat_t }{ x_t - z }
    &&\qquad\text{(subgradient ineq.)}
    \nonumber\\
    &~=~ \frac{1}{\eta_t} \inner{ \hat{y}_t - \hat{w}_{t+1} }{ x_t - z } 
    &&\qquad\text{(by \eqref{eq:Lazy})}\nonumber\\
    &~=~ \frac{1}{\eta_t} \Big( D_\Phi(x_t,w_{t+1}) 
    - D_\Phi(z,w_{t+1}) + \TriBreg{z}{x_t}{y_t} \Big),
    \EquationName{da_regret_at_each_round}
\end{align}
where in the last equation we have used \Proposition{GeneralBregmanTriangle} instead of \Proposition{BregmanTriangle}.
\end{mdframed}

As in the proof of \Theorem{somd}, the next step is to relate $D_\Phi(z,w_{t+1})$ to $D_\Phi(z,y_{t+1})$ so that \eqref{eq:da_regret_at_each_round} can be bounded using a telescoping sum.
The following claim is similar to \Claim{NewDSOMD}.

\begin{claim}
\ClaimName{NewDA}
Assume that $\gamma_t = \eta_{t+1}/\eta_t \in (0,1]$.
Then
\[
\eqref{eq:da_regret_at_each_round}
~\leq~
\frac{\TriBreg{x_t}{x_{t+1}}{w_{t+1}}}{\eta_t} + 
\underbrace{\Big( \frac{1}{\eta_{t+1}} - \frac{1}{\eta_t} \Big)}_{\text{telescopes}}
D_\Phi(z, x_1) +
\underbrace{\frac{\TriBreg{z}{x_t}{y_t}}{\eta_t} -\frac{\TriBreg{z}{x_{t+1}}{y_{t+1}}}{\eta_{t+1}}}_{\text{telescopes}}.
\]
\end{claim}
\begin{proof}
The first two steps are identical to the proof of \Claim{NewDSOMD}.
\begin{alignat*}{2}
&         \gamma_t \big(D_\Phi(z,w_{t+1}) - D_\Phi(x_{t+1},w_{t+1})\big)
\:+\: (1-\gamma_t) &&D_\Phi(z,x_1) \\
~\geq~&
         \gamma_t \TriBreg{z}{x_{t+1}}{w_{t+1}}
\:+\: (1-\gamma_t) \TriBreg{z}{x_{t+1}}{x_1}
    &&\text{(since $D_\Phi(x_{t+1},x_1)\geq 0$ and $\gamma_t \leq 1$)}\\
~=~&
\TriBreg{z}{x_{t+1}}{y_{t+1}}
&&\text{(by \Proposition{BregmanIdentity} and \eqref{eq:DAStabilization})}.
\end{alignat*}
Rearranging and using $\gamma_t > 0$ yields
\begin{equation}\EquationName{DARearrange}
D_\Phi(z,w_{t+1})
~\geq~
D_\Phi(x_{t+1},w_{t+1})
\:-\: \Big(\frac{1}{\gamma_t}-1\Big) D_\Phi(z,x_1)
\:+\:
\frac{\TriBreg{z}{x_{t+1}}{y_{t+1}}}{\gamma_t}.
\end{equation}
Plugging this into \eqref{eq:da_regret_at_each_round} yields
\begin{align*}
\eqref{eq:da_regret_at_each_round}
&~=~
\frac{1}{\eta_t} \Big( D_\Phi(x_t,w_{t+1}) 
    - D_\Phi(z,w_{t+1}) + \TriBreg{z}{x_t}{y_t} \Big)
\\
&
\begin{aligned}
    ~\leq~
    \frac{1}{\eta_t} \Bigg( & D_\Phi(x_t,w_{t+1})
    - D_\Phi(x_{t+1},w_{t+1}) + \mbox{}
    \\
     &\Big(\frac{1}{\gamma_t}-1\Big) D_\Phi(z,x_1) 
    - \frac{\TriBreg{z}{x_{t+1}}{y_{t+1}}}{\gamma_t}
    + \TriBreg{z}{x_t}{y_t}
     \Bigg),    
\end{aligned}
\qquad\text{by \eqref{eq:DARearrange}.}
\end{align*}
The claim follows by the definition of $\gamma_t$.
\end{proof}

\begin{mdframed}
Again the final step is very similar to the proof of \Theorem{somd}.
Summing \eqref{eq:da_regret_at_each_round} over $t$ and using \Claim{NewDA} leads to the desired telescoping sum.
\begin{align*}
&\sum_{t=1}^T \big( f_t(x_t) - f_t(z) \big)
\\
    ~\leq~ &\sum_{t=1}^T \Bigg(\frac{\TriBreg{x_t}{x_{t+1}}{w_{t+1}}}{\eta_t} + \Big( \frac{1}{\eta_{t+1}} - \frac{1}{\eta_t} \Big) D_\Phi(z, x_1) + \frac{\TriBreg{z}{x_t}{y_t}}{\eta_t} -\frac{\TriBreg{z}{x_{t+1}}{y_{t+1}}}{\eta_{t+1}} \Bigg)
    \\
    ~\leq~ &\sum_{t=1}^T \frac{\TriBreg{x_t}{x_{t+1}}{w_{t+1}}}{\eta_t}
      \:+\: \left( \frac{1}{\eta_1} + \sum_{t=1}^T \left( \frac{1}{\eta_{t+1}} - \frac{1}{\eta_t} \right)  \right){D_\Phi(z, x_1)} 
    \\
    ~=~ &\sum_{t=1}^T \frac{\TriBreg{x_t}{x_{t+1}}{w_{t+1}}}{\eta_t} \:+\: \frac{ D_\Phi(z, x_1) }{\eta_{T+1}},
\end{align*}
where for the second inequality we have also used that $\TriBreg{z}{x_1}{y_1}=D_\Phi(z,x_1)$ since $x_1=y_1$.
Thus, the above shows that
    \begin{equation}\EquationName{DAAlmostDone}
      \mathrm{Regret}(T, z) ~\leq~ \sum_{t=1}^T \frac{
      \TriBreg{x_t}{x_{t+1}}{w_{t+1}}}{\eta_t} + \frac{ D_\Phi(z, x_1) }{\eta_{T+1}}
      \quad \forall T > 0.
    \end{equation}
\end{mdframed}

Notice that \eqref{eq:DAAlmostDone} is syntactically identical to \eqref{eq:RegretBound};
the only difference is the definition of $w_{t+1}$ in these two settings.
However, in this section the definition of $x_t$ has not yet been used!
Doing so will provide an upper bound on \eqref{eq:DAAlmostDone}, which is the conclusion of this theorem.
To control $\TriBreg{x_t}{x_{t+1}}{w_{t+1}}$, we will apply \Proposition{GeneralBregmanProjection} as follows.
Taking $p = y_t$, $\pi = x_t = \Pi_\cX^\Phi(y_t)$, $v=x_{t+1}$ and $\hat{q} = \eta_t g_t$, we obtain
\begin{align*}
    \TriBreg{x_t}{x_{t+1}}{w_{t+1}}
        &~=~ -\TriBreg{v}{\pi}{\Grad\Phi^*(\hat{p} - \hat{q})}
            \qquad\text{(since $\hat{w}_{t+1} = \hat{y}_t - \eta_t \ghat_t = \hat{p}-\hat{q}$)} \\
        &~\leq~ -\TriBreg{v}{\pi}{\Grad\Phi^*(\hat{\pi} - \hat{q})}
            \qquad\text{(by \Proposition{GeneralBregmanProjection})} \\
        &~=~ \TriBreg{x_t}{x_{t+1}}{\Grad\Phi^*(\hat{x}_t - \eta_t \ghat_t)}.
\end{align*}
Plugging this into \eqref{eq:DAAlmostDone} completes the proof.
\end{proofof}

\subsection{Remarks}

Interestingly, the \emph{doubling trick}~\citep{Shalev-Shwartz12} on OMD can be
viewed as an incarnation of stabilization. To see this, set $\eta_t \coloneqq
1/\sqrt{2^{\lfloor\lg t \rfloor} }$ and $\gamma_t \coloneqq
\textbf{1}_{\{\mbox{\small$t$ is a power of 2}\}} $. Then, for each dyadic
interval of length $2^\ell$, the first iterate is $x_1$ and a fixed learning
rate $1/\sqrt{\smash[b]{2^\ell}}$ is used. Thus, with these parameters,
Algorithm~\ref{alg:somd} reduces to the doubling trick.

One should note that in Theorem~\ref{thm:somd} the stabilization parameter
\(\gamma_t\) used in round \(t \geq 1\) depends on the learning rates $\eta_t$ and $\eta_{t+1}$, the latter of which is used in the \emph{next} round.
So the learning rate $\eta_{t+1}$ must be known in iteration $t$ in order to calculate $\gamma_t$ appropriately.
This subtlety will arise, for example, when we derive first-order regret bounds in \Section{FirstOrder} --- here the learning rate is based on the subgradients of the
past functions, not just on the iteration number.
Reindexing the learning rates could fix the problem, but then the proof of \Theorem{somd} would look syntactically odd.
Although this dependence on the future may seem unnatural, in \Section{comparison} we shall see that
under some mild conditions, stabilized OMD coincides exactly with DA with
dynamic learning rates after reindexing. This matches the behavior observed
between OMD and DA when the learning rates are fixed. In this sense,
stabilization may seem as a natural way to fix OMD for dynamic learning rates.

\section{Applications}
\label{sec:applications}

In this section we show that stabilized-OMD and DA enjoy the same regret bounds in several applications that involve a dynamic learning rate.

\subsection{Strongly-convex mirror maps}
\SectionName{strongly-convex-ds-omd}

We now analyze the algorithms of the previous section in the scenario
that the mirror maps are strongly convex. Let $\eta_t, \gamma_t, f_t$ be
as in the previous section. The next result is a corollary of
Theorems~\ref{thm:somd}, \ref{thm:psomd}, and \ref{thm:DualAveraging}.

\begin{corollary}[Regret bound for stabilized OMD and DA]
\CorollaryName{Nesterov} Suppose that the mirror map $\Phi$ is $\rho$-strongly convex
on $\cX$ with respect to a norm $\| \cdot \|$. Let $\{x_t\}_{t \geq 1}$
be the iterates produced by one of Algorithms~\ref{alg:somd},~\ref{alg:ps_omd}, or~\ref{alg:DualAveraging} (for \Algorithm{ps_omd}, the
additional assumption \eqref{eq:PSOMDConvexity} is required). Then, for
all~\(T > 0\) and \(z \in \cX\),
\begin{equation*}
  \mathrm{Regret}(T, z) ~\leq~ \sum_{t=1}^T \frac{\eta_t
    \|\ghat_t\|_*^2}{2\rho} 
    +  \frac{ D_\Phi(z, x_1) }{\eta_{T+1}}.
\end{equation*}
\end{corollary}

This is identical to the bound for dual averaging in
\citet[Eq.~2.15]{Nesterov09} (taking his $\lambda_i \coloneqq 1$ and his
$\beta_i \coloneqq1/\eta_i$).
%
%
\comment{
This situation is not too surprising: with a fixed learning rate
$\eta$, OMD (without stabilization) and DA are known to have an
identical regret bound.  Nevertheless, one may verify that the
algorithms are genuinely different and generate different sequences of
iterates.
}
The proof is based on the following simple proposition, which bounds the
Bregman divergence when $\Phi$ is strongly convex \citep[pp.~300]{Bubeck15}. 

\begin{restatable}{proposition}{StrongConvexBound}
\PropositionName{StrongConvexBound}
Suppose $\Phi$ is $\rho$-strongly convex on $\cX$ with respect to $\norm{\cdot}$.
For any $x, x' \in \cX$ and $\hat{q} \in \bR^n$,
\begin{equation*}
    \TriBreg{x}{x'}{\Grad\Phi^*(\hat{x} - \hat{q})
} \:\leq\: \norm{\hat{q}}_*^2 / 2 \rho.
\end{equation*}
\end{restatable}

\begin{proof}
    First we apply \Proposition{BregmanTriangle} with $a=x$, $b=x'$ and $d=\Grad\Phi^*(\hat{x}-\hat{q})$ to obtain
    \begin{alignat*}{3}
    \TriBreg{x}{x'}{d}
    &~=~ \inner{\hat{x}-\hat{d}}{x-x'} - D_\Phi(x',x) \\
    &~=~ \inner{\hat{q}}{x-x'} - D_\Phi(x',x)
        &&\quad\text{(since $\hat{d}=\hat{x}-\hat{q}$)}\\
    &~\leq~ \norm{\hat{q}}_* \norm{x-x'} - \frac{\rho}{2}\norm{x-x'}^2
        &&\quad\text{(by the def.\ of dual norm and Prop.~\ref{prop:BregmanStrong})}\\
    &~\leq~ \norm{\hat{q}}_*^2 / 2 \rho
        &&\quad\text{(by \Fact{quadratic})}.
    \end{alignat*}
\end{proof}

Now the proof of \Corollary{Nesterov} is a simple application of the above proposition to the abstract regret bounds we have for each algorithm.

\vspace{6pt}
\begin{proofof}{\Corollary{Nesterov}}
The regret bounds proven by Theorems~\ref{thm:somd}, \ref{thm:psomd} and \ref{thm:DualAveraging} all involve a summation with terms of a similar form.
\begin{align*}
&\text{\Theorem{somd}:}~~\TriBreg{x_t}{x_{t+1}}{w_{t+1}} \\
&\text{\Theorem{psomd}:}~~\TriBreg{x_t}{y_{t+1}}{w_{t+1}}
\\
&\text{\Theorem{DualAveraging}:}~~\TriBreg{x_t}{x_{t+1}}{\Grad\Phi^*(\hat{x}_t - \eta_t \ghat_t)}
\end{align*}
We may bound all three using \Proposition{StrongConvexBound}.
In all three cases we have $x_t, x_{t+1} \in \cX$.
In \Theorem{psomd} we additionally have $y_{t+1} \in \cX$.
For Theorems~\ref{thm:somd} and \ref{thm:psomd}
we have $w_{t+1} = \Grad\Phi^*(\hat{x}_t - \eta_t \ghat_t)$ by
\eqref{eq:GradientStep} and \eqref{eq:GradientStep2}. Therefore all of these
terms may be bounded using \Proposition{StrongConvexBound} with $x=x_t$
and $\hat{q} = \eta_t \ghat_t$, yielding the claimed bound.
\end{proofof}

\subsection{Prediction with expert advice}
\SectionName{rwm-ds-omd}

Next consider the setting of prediction with expert advice. In this
setting $\cX$ is the simplex $\Delta_n
\subset \bR^n$, the mirror map is $\Phi(x) \coloneqq \sum_{i=1}^n
x_i \log x_i$ for all $x \in \bar{\cD} \coloneqq \bR_{\geq 0}^n$ (taking $0 \ln 0 = 0$). Note that on $\cX$ the mirror map $\Phi$ is the negative of the entropy
function. The gradient of the mirror map and its conjugate are
\begin{equation}
\EquationName{ExpertMirror}
\Grad\Phi(x)_i = \ln(x_i) + 1
\quad \text{and} \quad
\Grad\Phi^*(\hat{x})_i = e^{\hat{x}_i-1}, \qquad \forall x \in \cD, \forall \xhat \in \Reals^n.
\end{equation}
For any two points $a \in \bar{\cD}$ and $b \in \cD$, an easy
calculation shows that $D_\Phi(a,b)$ is the generalized KL-divergence
$$
D_\KL(a,b) ~=~ \sum_{i=1}^n a_i \ln(a_i/b_i) - \norm{a}_1 + \norm{b}_1.
$$
Note that the KL-divergence is convex in its second argument for any \(b
\in \Dcal = \bR_{> 0}^n\) since the functions \(-\ln(\cdot)\) and
absolute value are both convex. This means that all the abstract regret
bounds from \Section{StabilizedOMD} hold in this setting. Using them we
will derive regret bounds for this setting with a little extra-work. As
an intermediate step, we will derive bounds that use the following
function:
\begin{equation*}
\Lambda(a,b)
~\coloneqq~ D_\KL(a,b)
 + \norm{a}_1 - \norm{b}_1 + \ln \norm{b}_1 
~=~ \sum_{i=1}^n a_i \ln(a_i/b_i) + \ln \norm{b}_1,
\end{equation*}
which is a standard tool in the analysis of algorithms for
the experts' problem. For examples, see~\citet[\S 2.1]{RooijEGK14a} and
\citet[Lemma~4]{cesa2007improved}. 
\comment{
An initial observation shows that
$\Lambda$ is non-negative in the experts' setting.
\begin{proposition}
\PropositionName{LambdaNonNeg}
$\Lambda(a,b) \geq 0$ for all $a \in \cX$, $b \in \cD$.
\end{proposition}
\comment{
\begin{proof}
Let us write $\Lambda(a,b) = - \sum_{i=1}^n a_i \ln \frac{b_i}{a_i} +
\ln\big(\sum_{i=1}^n b_i\big)$. Since $a$ is a probability distribution,
we may apply Jensen's inequality to show that this expression is
non-negative.
\end{proof}}
}

The next result is a corollary of
Theorems~\ref{thm:somd}, \ref{thm:psomd}, and \ref{thm:DualAveraging}.

\begin{corollary}
\CorollaryName{better_srwm_general} 
Assume we have $\eta_t \geq \eta_{t+1} > 0$ for each~$t > 1$. Define $\gamma_t \coloneqq \eta_{t+1}/\eta_t \in
(0,1]$ for all $t \geq 1$. Let \(x_1 \coloneqq \frac{1}{n}\vec{1}\) be the uniform distribution and $\{x_t\}_{t \geq 2}$
be the iterates produced by one of Algorithms~\ref{alg:somd},~\ref{alg:ps_omd}, or~\ref{alg:DualAveraging} in the setting of prediction with expert advice. Then, for all \(T > 0\) and \(z \in \cX\),
\begin{equation}
\EquationName{RWMRegret}
\Reg(T,z) ~\leq~ \sum_{t=1}^T  \frac{\Lambda(x_t,\Grad\Phi^*(\hat{x}_t - \eta_t \ghat_t))}{\eta_t}
     + \frac{\ln n}{\eta_{T+1}}.
\end{equation}
\end{corollary}

The proof is a direct consequence of the following proposition, which is proven in \Appendix{rwm-ds-omd}.
\begin{restatable}{proposition}{dualsomdKL}
\PropositionName{dual_somd_KL}
$ \TriBreg{a}{b}{c} \leq \Lambda(a,c) $ for $a, b \in \cX$,  $c \in \cD$.
\end{restatable}
\begin{proofof}{\Corollary{better_srwm_general}}
Recall that \(D_{\KL}\) is convex in its second
argument, which allows us to use the bound~\eqref{eq:PSRegretBound}
for primal-stabilized OMD. As in the proof of \Corollary{Nesterov}, we
first observe that the regret bounds
\eqref{eq:RegretBound}, \eqref{eq:PSRegretBound} and \eqref{eq:DARegretBound}
all have sums with terms of the form
$\TriBreg{x_t}{u_t}{\Grad\Phi^*(\hat{x}_t - \eta_t \ghat_t)}$ for some irrelevant $u_t
\in \cX$, and hence may be bounded using \Proposition{dual_somd_KL}. Finally,
the standard inequality $\sup_{z \in \cX} D_\KL(z, x_1) \leq \ln n$
completes the proof.
\end{proofof}

\subsubsection{Anytime regret} \label{sec:anytime-regret}

From \Corollary{better_srwm_general} we now derive an anytime regret
bound in the case of bounded costs. This matches the best known bound
appearing in the literature; see \citet[Theorem
2.4]{introduction-online-optimization} and \citet[Proposition 2.1]{Ger11}.
\TRJournalSplit{Moreover, in Appendix~\ref{sec:app_lowerbound} we show that this is tight for DA in the case $n=2$.}{ Moreover, in our technical report~\citep{TR} we show that this is tight for DA in the case $n=2$.}
By the equivalence of DA and DS-OMD in the experts' setting (\Corollary{EquivDSOMDAndDA} in \Section{comparison}), this regret bound is also tight for DS-OMD.

\begin{corollary}
\CorollaryName{better_srwm} Define $\eta_t = 2 \sqrt{\ln(n) / t}$ and
$\gamma_t = \eta_{t+1}/\eta_t \in (0,1]$ for each $t \geq 1$. Let $\{f_t
\coloneqq \iprod{\ghat_t}{\cdot}\}_{t \geq 1}$ be such that $\ghat_t \in
[0,1]^n$ for all~\(t \geq 1\). Let \(x_1\) be the uniform distribution
$\frac{1}{n}\vec{1}$ and let $\{x_t\}_{t \geq 2}$ be as in one of
Algorithms~\ref{alg:somd},
\ref{alg:ps_omd}, or~\ref{alg:DualAveraging} in the prediction with experts advice setting. Then,
\begin{equation*}
\EquationName{AnytimeRWMRegret}
\Regret(T,z) ~\leq~ \sqrt{ T \ln n },
\qquad\forall T \geq 1, \forall z \in \X.
\end{equation*}
\end{corollary}


The proof follows from 
\Corollary{better_srwm_general} and Hoeffding's Lemma, as shown below.

\begin{lemma}[{\protect Hoeffding's Lemma,~\citealp[Lemma 2.2]{game_prediction}}]
\LemmaName{hoeffding}
Let $X$ be a random variable with $a \leq X \leq b$. Then for any $s \in \R$,
\[
\ln \E [ e^{sX} ] - s \E X ~\leq~ \frac{s^2(b-a)^2}{8}.
\]
\end{lemma}

\begin{proofof}{\Corollary{better_srwm}}
By~\eqref{eq:ExpertMirror} we have \(\nabla \Phi^*(\hat{x}_t - \eta_t \ghat_t)_i =
x_t(i) \exp(-\eta_t \ghat_t(i))\) for each \(i \in [n]\). This together with
\Lemma{hoeffding} for $s = -\eta_t$ yields
\begin{align*}
    \label{eq:srwm2a}
    \Lambda(x_t&, \nabla \Phi^*(\hat{x}_t - \eta_t \ghat_t)) 
    = \eta_t \iprod{\ghat_t}{x_t} + \ln\paren[\Big]{
        \sum_{i=1}^n x_{t}(i) e^{-\eta_t \ghat_{t}(i)}
    } 
    \leq \frac{\eta_t^2}{8}.
 \end{align*}
Plugging this and $\eta_t = 2 \sqrt{ \frac{\ln n}{t} }$ into \Equation{RWMRegret}, we obtain
\begin{equation*}
\Regret(T)
\leq
\sqrt{\ln n} \Bigg( \frac{1}{4} \sum_{t=1}^T \frac{1}{\sqrt{t}} + \frac{\sqrt{T+1}}{2} \Bigg)
\leq
\sqrt{\ln n} \Bigg( \frac{2\sqrt{T} - 1}{4} + \frac{\sqrt{T} + 0.5}{2} \Bigg) 
\leq \sqrt{T \ln n}
\end{equation*}
by \Fact{SquareRoot} and concavity of square root.
\end{proofof}

\subsubsection{First-order regret bound}
\SectionName{FirstOrder}

The regret bound described in Section~\ref{sec:anytime-regret} depends
on~$\sqrt{T}$; this is known as a ``zeroth-order'' regret bound. In some
scenarios the cost of the best expert up to time $T$ can be far less
than $T$. This makes the problem somewhat easier, and it is possible to
improve the regret bound. Formally, let $L_T^*$ denote the total cost of
the best expert until time~$T$. Then $L_T^* \leq T$ due to our
assumption that all costs are at most~$1$. A ``first-order'' regret
bound depends on $\sqrt{L_T^*}$ instead of~$\sqrt{T}$.

The only modification to the algorithm is to change the learning rate.
If the costs are ``smaller than expected'', then intuitively time is
progressing ``slower than expected''. We will adopt an elegant idea from
\citet{AuerCG02}, which is to use the algorithm's cost itself as a
measure of the progression of time, and to incorporate this into the
learning rate. They called this a ``self-confident'' learning rate.

\begin{corollary}
\CorollaryName{first-order-bound} Let $\{f_t \coloneqq
\iprod{\ghat_t}{\cdot}\}_{t \geq 1}$ be such that $\ghat_t \in [0,1]^n$ for
all~\(t \geq 1\). Define $\gamma_t = \eta_{t+1}/\eta_t \in (0,1]$  and
$\eta_t = \sqrt{\ln(n) / (1+\sum_{i<t} \inner{\ghat_i}{x_i}) }$ for all $t
\geq 1$.  Let \(x_1\) be the uniform distribution
$\frac{1}{n}\vec{1}$ and let $\{x_t\}_{t \geq 2}$ be as in one of
Algorithms~\ref{alg:somd},
\ref{alg:ps_omd}, or~\ref{alg:DualAveraging} in the prediction with experts advice setting. Denote the minimum
total cost of any expert up to time $T$ as $L_T^* \coloneqq \min_{j \in
[n]} \sum_{t=1}^T \ghat_{t}(j) $. Then,
\[
\Reg(T,z) ~\leq~ 2\sqrt{ \ln(n) L_T^* } +  8 {\ln n},
\quad\forall T \geq 1, \forall z \in \cX.
\]
\end{corollary}





The main ingredient is the following alternative bound on~$\Lambda$,
which is proven in \Appendix{rwm-ds-omd}.

\begin{restatable}{proposition}{LambdaTaylor}
\PropositionName{LambdaTaylor}
Let $a \in \cX$, $\hat{q} \in [0,1]^n$ and $\eta>0$.
Then $\Lambda(a,\Grad\Phi^*(\hat{a}-\eta\hat{q})) \leq \eta^2 \inner{a}{\hat{q}}/2$.
\end{restatable}

\begin{proofof}{\Corollary{first-order-bound}}
Let \(z \in \X\). From \Corollary{better_srwm_general} and \Proposition{LambdaTaylor}, we have
\begin{align}
    \sup_{z' \in \cX}\sum_{t=1}^T\inner{\ghat_t}{x_t - z'} 
     ~\leq~ \sum_{t=1}^T \frac{\eta_t \inner{\ghat_t}{x_t}}{2}
     + \frac{\ln n}{\eta_{T+1}}. ~\label{eq:first-order-eq2}
\end{align}
Denote the algorithm's total cost at time $t$ by $A_t = \sum_{i \leq t}
\inner{\ghat_i}{x_i}$. Recall that the total cost of the best expert at time
$T$ is $L_T^* = \min_{z' \in \Delta_n} \sum_{t=1}^T \inner{\ghat_t}{z'}$
and the learning rate is $\eta_t = \sqrt{ \ln(n) / (1+A_{t-1})}$.
Substituting into~\eqref{eq:first-order-eq2},
\begin{align*}
    A_T - L_T^*
    ~   \leq~ \sqrt{\ln n} \Bigg( \frac{1}{2}
     \sum_{t=1}^T \frac{\inner{\ghat_t}{x_t}}{\sqrt{ 1+A_{t-1} }} ~+~ \sqrt{1+A_T} \Bigg)
    ~\leq~ \sqrt{ \ln n } \Big( \sqrt{ A_T } ~+~ \sqrt{A_T} + 1 \Big)
\end{align*}
by Proposition~\ref{prop:at_sum}
with $a_i = \inner{\ghat_i}{x_i}$ and $u = 1$.
Rewriting the previous inequality, we have shown that
\begin{equation*}
    A_T - L_T^*
    \leq 2 \sqrt{ \ln(n) A_T } + \sqrt{\ln n}.    
\end{equation*}
By Proposition~\ref{prop:ineq_xy} we obtain
\begin{equation*}
A_T - L_T^* 
    \leq 2 \sqrt{\ln(n) L_T^* } + \sqrt{\ln n} + 2\left( \ln n \right)^{3/4} + 4{\ln n}.
\end{equation*}
Since $A_T - L_T^* \geq \mathrm{Regret}(T,z)$, the result follows.
\end{proofof}

Comparing our bound with some existing results in the literature: our
constant term of $2$ obtained in \Corollary{first-order-bound} is better
than the constant ($\sqrt{2}/(\sqrt{2}-1)$) obtained by the doubling
trick~\citep[Exercise~2.8]{game_prediction}, and the constant
($2\sqrt{2}$) in~\citet{AuerCG02} but worse than the constant
($\sqrt{2}$) of the best known first-order regret bound due
to~\citet{YaroshinskyES04}. We also match the constant $2$ of the Hedge algorithm
from~\citet[Theorem~8]{RooijEGK14a}. Their result is actually more
general; we could similarly generalize our analysis, but that would
deviate too far from the main purpose of this paper.


\section{Comparing DS-OMD and DA}
\SectionName{comparison}

In this section we will write the iterates of dual-stabilized OMD in
two equivalent forms. First we will write it in a proximal-like
formulation similar to the mirror descent formulation
by~\citet{BeckT03}, shedding some light into the intuition behind
dual-stabilization. We then write the iterates from DS-OMD in a form
very similar to the original definition of DA by~\citet{Nesterov09}. 
This will allow us to intuitively understand why OMD does has poor results with dynamic step-size and to derive simple sufficient conditions under
which DS-OMD and DA generate the same iterates, mimicking  the relation
between OMD and DA for a fixed learning rate.

\citet{BeckT03} showed that the iterate \(x_{t+1}\) for round \(t+1\) from
OMD is the unique minimizer over \(\X\) of \(\eta_t
\iprod{\ghat_t}{\cdot} + D_{\Phi}(\cdot, x_t)\), where \(\ghat_t \in
\subdiff[f_t](x_t)\). The next proposition extends this formulation to
DS-OMD, recovering the result from Beck and Teboulle when~\(\gamma_t =
1\).  The proof is a simple application of optimality conditions
of~\eqref{eq:ds-omd_proximal}. (For the sake of completeness we carefully state classical results on optimality
conditions in \Appendix{comparison_app}). Recall that the \textbf{normal cone} to a set $C \subseteq \Reals^n$ at $x \in \Reals^n$ is the set $N_C(x)
\coloneqq \{ s\in \R^n \mid \langle s, y - x \rangle \leq 0~\forall y \in C\}$.

\begin{restatable}{proposition}{PropDsomdProx}
    \label{prop:prox-ds-omd}
    Let $\{f_t\}_{t \geq 1}$ be a sequence of convex functions with $f_t
\colon \X \rightarrow \R$ for each \(t \geq 1\). Assume we have $\eta_t \geq \eta_{t+1} > 0$ and \(\gamma_t \in [0,1]\) for each~$t \geq 1$.
    Let $\{x_t\}_{t \geq 1}$  and \(\{\ghat_t\}_{t \geq 1}\) be as in
    Algorithm~\ref{alg:somd}. Then, for any \(t \geq 1\),
    \begin{equation}
        \label{eq:ds-omd_proximal}
            \{x_{t+1}\} 
        = \argmin_{x \in \X}
        \paren[\Big]{ \gamma_t\paren[\big]{\eta_t \iprod{\ghat_t}{x}
        + D_{\Phi}(x, x_t)} + (1 - \gamma_t) D_{\Phi}(x, x_1)}.
    \end{equation}
\end{restatable}
\begin{proof}
    Let \(t \geq 1\) and let \(F_t \colon \Dcal \to \Reals\) be the
    function being minimized on the right-hand side
    of~\eqref{eq:ds-omd_proximal}. By definition we have \(x_{t+1} =
    \Pi_{\X}^\Phi({y}_{t+1})\).
    By the optimality conditions for Bregman projections (see
    Lemma~\ref{lemma:projection} in \Appendix{comparison_app}),
    \begin{equation*}
        x_{t+1} = \Pi_{\X}^\Phi({y}_{t+1})
        \iff   \yhat_{t+1} - \xhat_{t+1}
         = -\nabla (D_{\Phi}(\cdot, y_{t+1}))(x_{t+1})\in N_\X(x_{t+1}).
    \end{equation*}
    Referring to \eqref{eq:BregmanDef} we see that
    \(\nabla (D_{\Phi}(\cdot, b))(a) = \hat{a} - \hat{b}\) for any \(a,b \in \Dcal\). 
    Using this and the definitions
    from Algorithm~\ref{alg:somd} we get
    \begin{align*}
        \hat{y}_{t+1} - \xhat_{t+1}
        &= \gamma_t(\hat{x}_t - \eta_t \ghat_t)
        + (1 - \gamma_t) \hat{x}_1 - \xhat_{t+1}\\
        %
        %
        &=  \gamma_t(\xhat_t - \xhat_{t+1} 
        - \eta_t \ghat_t)
        + (1 - \gamma_t) (\xhat_1 - \xhat_{t+1}) \\
        &= - \gamma_t\paren[\big]{\nabla(D_\Phi(\cdot, x_{t}))(x_{t+1})  
        + \eta_t \ghat_t}
        - (1 - \gamma_t) \nabla(D_\Phi( \cdot, x_1 ) )(x_{t+1}) \\
        &= - \nabla F_t(x_{t+1}).
    \end{align*}
    Thus, we have \(- \nabla F_t(x_{t+1}) \in N_{\X}(x_{t+1}) \). Again by
    classical optimality conditions for convex optimization we
    conclude that \(x_{t+1} \in \argmin_{x \in \X} F_t(x)\), and strict convexity of \(\Phi\) yields uniqueness of the minimizer,
    as desired.
\end{proof}
As expected, when \(\gamma_t = 1\) for each \(t \geq 1\) the above
theorem recovers exactly the OMD formulation from~\citet{BeckT03}. Thus, the above result provides intuition to understand the effect of the
stabilization step on the iterates of the algorithm: it biases the
iterates toward points in \(\X\) which are not too far (w.r.t.\
the Bregman Divergence) from \(x_1\).


Despite their similar descriptions, \citet{OrabonaP18} showed that
OMD and DA may behave in extremely different ways even on the
well-studied experts' problem with similar choices of step-sizes. This
extreme difference in behavior is not clear from the classical
algorithmic description of these methods as in \Algorithm{aomd}.

It is also well-known that DA can be seen as an instance of the
FTRL algorithm; see~\citet[\S 4.4]{Bubeck15} or~\citet[\S
5.3.1]{Hazan16}. More specifically, if \(\{x_t\}_{t \geq 1}\) and
\(\{\ghat_t\}_{t \geq 1}\) are as in \Algorithm{DualAveraging}, then for every \(t \geq 0\), we have\footnote{The
\(\iprod{\nabla \Phi(x_1)}{x}\) term disappears if \(x_1\) minimizes
$\Phi$ on~\(\X\).}
\begin{equation}
    \label{eq:da_min_form}
    \tag{DA-Prox}
    \{x_{t+1}\} = \argmin_{x \in \X}
    \paren[\Big]{\eta_{t+1} \sum_{i = 1}^t \iprod{\ghat_i}{x} 
    - \iprod{\xhat_1}{x} + \Phi(x)}.
\end{equation}
\comment{The above closed formula for the iterates is advantageous when
we want to compare the iterates of DA with iterates from other
algorithms since it depends only on the subgradients \(\ghat_t\) and not
directly on the iterates \(x_t\).} 
In the next theorem, proven in~\Appendix{comparison_app}, we write
DS-OMD in a similar form, but with vectors from the normal cone of
\(\X\) creeping into the formula due to repeatedly mapping  between the
primal and dual spaces. The result
by~\citet[Theorem~11]{McMahan17} is similar but slightly more intricate
due to the use of time-varying mirror maps. Moreover, their result does
not directly apply when we have stabilization. 

\begin{restatable}{theorem}{AvgMinForm}
    \label{thm:avg_min_form} 
    Let $\{f_t\}_{t \geq 1}$ with $f_t : \X \rightarrow \R$ be a
    sequence of convex functions and let \(\eta \colon \Naturals \to
    \Reals_{>0}\) be non-increasing. Let $\{x_t\}_{t \geq 1}$ and
    \(\{\ghat_t\}_{t \geq 1}\) be as in Algorithm~\ref{alg:somd}. Then,
    there are $\{p_t\}_{t \geq 1}$ with $p_t \in N_{\X}(x_t)$ for all $t
    \geq 1$ such that, if~\(\gamma_i = 1\) for all \(i \geq 1\), then
    for all \(t \geq 0\)
    \begin{equation}
        \label{eq:omd_min_form}
        \tag{OMD-Prox}
        \{x_{t + 1}\}
        = \argmin_{x \in \X}
        \paren[\Big]{\sum_{i = 1}^t \iprod{\eta_i \ghat_i + p_i}{x} - \iprod{\hat{x}_1}{x} + \Phi(x)}
    \end{equation}
    and if~\(\gamma_i = \tfrac{\eta_{i+1}}{\eta_i}\) for all \(i \geq 1\), then for all \(t \geq 0\)
    \begin{equation}
        \label{eq:ds-omd_min_form_2}
        \tag{DSOMD-Prox}
        \{x_{t + 1}\}
        = \argmin_{x \in \X}
        \paren[\Big]{\eta_{t+1} \sum_{i = 1}^t \iprod{\ghat_i + p_i'}{x} - \iprod{\hat{x}_1}{x} + \Phi(x)}.
    \end{equation}
    where \(p_t' \coloneqq \frac{1}{\eta_t} p_t \in N_\X(x_t)\) for every \(t \geq 1\).
\end{restatable}

\Theorem{avg_min_form} is an easy consequence of the following proposition.

\begin{proposition}
    \PropositionName{general_avg_min_form}
    Let $\{f_t\}_{t \geq 1}$ with $f_t \colon \X \rightarrow \R$ be a
    sequence of convex functions and let \(\eta \colon \Naturals \to \Reals_{>0}\) be non-increasing. Let $\{x_t\}_{t \geq 1}$ and
    \(\{\ghat_t\}_{t \geq 1}\) be as in Algorithm~\ref{alg:somd}.
    Define $\gamma^{[i,t]} \coloneqq \prod_{j = i}^t \gamma_j$
    for every~\(i,t  \in \Naturals\) with the convention \(\prod_{j = i}^t \gamma_j = 1\) for \(t < i\). Then, there are $\{p_t\}_{t \geq 1}$ with $p_t \in N_{\X}(x_t)$ for each $t \geq 1$ such that
    \begin{equation}
        \label{eq:ds-omd_min_form}
        \{x_{t + 1}\}
        = \argmin_{x \in \X}
        \paren[\Big]{\sum_{i = 1}^t \gamma^{[i,t]}\iprod{\eta_i \ghat_i + p_i}{x} - \paren[\Big]{\gamma^{[1,t]} + \sum_{i = 1}^t \gamma^{[i+1, t]} (1 - \gamma_i)} \iprod{\xhat_1}{x} + \Phi(x)}, 
        \quad \forall t \geq 0.
    \end{equation}    
\end{proposition}
\begin{proof}
    First of all, in order to prove~\eqref{eq:ds-omd_min_form} we claim it
    suffices to prove that there are \(\{p_t\}_{t \geq 1}\) with \(p_t \in
    N_{\X}(x_t)\) for each \(t \geq 1\) such that 
    \begin{equation}
        \label{eq:ds-omd_min_form_proof_1}
        \yhat_{t+1} 
        = - \sum_{i = 1}^t \gamma^{[i,t]}(\eta_i \ghat_i + p_i)
        + \paren[\Big]{\gamma^{[1,t]} 
        + \sum_{i = 1}^t \gamma^{[i+1,t]} (1 - \gamma_i)} \xhat_1,
        \qquad \forall t \geq 0. 
    \end{equation}
    To see the sufficiency of this claim, note that
    \begin{alignat*}{3}
        x_{t+1} &= \Pi_\X^{\Phi}(y_{t+1})
        \\
        &\iff \yhat_{t+1} - \xhat_{t+1} \in N_\X(x_{t+1})&
         \text{(by Lemma~\ref{lemma:projection})}
        \\
        &\iff \yhat_{t+1}  \in \subdiff[(\Phi 
        + {\indic[\X]{\cdot}})](x_{t+1})
        & (\subdiff[({\indic[\X]{\cdot}})](x) = N_{\X}(x))
         \\
        &\iff x_{t+1} \in \argmax_{x \in \Reals^n}
        \paren[\big]{\iprod{\yhat_{t+1}}{x} 
        - \Phi(x) - {\indic[\X]{x}}}
        & (\text{by Lemma~\ref{lemma:subdiff_attainability}})\\
        & \iff x_{t+1} \in \argmin_{x \in \X}\paren[\big]{-\iprod{\yhat_{t+1}}{x} 
        + \Phi(x)}.
    \end{alignat*}
    The above together with~\eqref{eq:ds-omd_min_form_proof_1}
    yields~\eqref{eq:ds-omd_min_form}. Let us now
    prove~\eqref{eq:ds-omd_min_form_proof_1} by induction on \(t \geq 0\).

    For \(t = 0\), we have that~\eqref{eq:ds-omd_min_form_proof_1} holds trivially.
    Let \(t > 0\). By definition, we have \(\yhat_{t+1} = (1 -
    \gamma_t)(\xhat_t - \eta_t \ghat_t) + \gamma_t \xhat_1\). At this point,
    to use the induction hypothesis, we need to write \(\xhat_t\) as a
    function of \(\yhat_t\). From the definition of
    Algorithm~\ref{alg:somd}, we have \(x_t = \Pi_{\X}^{\Phi}(y_t)\). By
    Lemma~\ref{lemma:projection}, the latter holds if and only if
    \(\yhat_t - \xhat_t  \in N_\X(x_t)\). That is, there is \(p_t \in
    N_\X(x_t)\) such that \(\xhat_t = \yhat_t - p_t\). Plugging these
    facts together and using our induction hypothesis we have
    \begin{align*}
        \yhat_{t+1} 
        &= \gamma_t(\xhat_t - \eta_t \ghat_t) + (1 - \gamma_t) \xhat_1
        = \gamma_t(\yhat_t - \eta_t \ghat_t - p_t) + (1 - \gamma_t) \xhat_1\\
        &\stackrel{\text{I.H.}}{=}
        \begin{aligned}[t]
            \gamma_t\paren[\Bigg]{
                -\sum_{i = 1}^{t-1} &\gamma^{[i,t-1]}(\eta_i \ghat_i + p_i)
                - \eta_t \ghat_t - p_t  \\
                &+ \paren[\Big]{\gamma^{[1,t-1]} + \sum_{i = 1}^{t-1} \gamma^{[i+1,t-1]} (1 - \gamma_i) } \xhat_1
            }
            + (1 - \gamma_t )\xhat_1 
        \end{aligned}
        \\
        &= -\sum_{i = 1}^{t} \gamma^{[i,t]}(\eta_i \ghat_i + p_i)
             + \paren[\Big]{\gamma^{[1,t]} 
            + \sum_{i = 1}^{t} \gamma^{[i+1,t]} (1 - \gamma_i) } \xhat_1,
    \end{align*}
    and this finishes the proof of~\eqref{eq:ds-omd_min_form_proof_1}.
\end{proof}

\begin{proofof}{\Theorem{avg_min_form}}
    Define $\gamma^{[i,t]}$
    for every~\(i,t  \in \Naturals\) as in \Proposition{general_avg_min_form}. If \(\gamma_t = 1\) for all \(t \geq 1\), then \(\gamma^{[i,t]} =
    1\) for any \(t,i \geq 1\). Moreover,  if~\(\gamma_t = 
    \tfrac{\eta_{t+1}}{\eta_t}\) for every \(t \geq 1\), then for every \(t, i
    \in \Naturals\) with \(t \geq i\) we have \(\gamma^{[i,t]} =
    \tfrac{\eta_{t+1}}{\eta_i}\), which yields \(\gamma^{[i,t]}(\eta_i \ghat_i + p_i) =
    \eta_{t+1} (\ghat_i + \tfrac{1}{\eta_i} p_i)\) and
    \begin{align*}
        \gamma^{[1,t]} + \sum_{i = 1}^t \gamma^{[i+1,t]} (1 - \gamma_i)
        &= \frac{\eta_{t+1}}{\eta_1} 
    + \sum_{i = 1}^t \frac{\eta_{t+1}}{\eta_{i+1}} 
        \paren[\Big]{1 - \frac{\eta_{i+1}}{\eta_i}}\\
        &= \frac{\eta_{t+1}}{\eta_1} 
        + \eta_{t+1} \sum_{i = 1}^t 
        \paren[\Big]{\frac{1}{\eta_{i+1}}  - \frac{1}{\eta_i}}
        = 1. 
        \end{align*}
\end{proofof}

With the above theorem, we may compare the iterates of DA, OMD, and
DS-OMD by comparing the
formulas~\eqref{eq:da_min_form},~\eqref{eq:omd_min_form},
and~\eqref{eq:ds-omd_min_form_2}. For the simple unconstrained case
where \(\X = \Reals^n\) we have \(N_{\X}(x_t) = \{0\}\) for each \(t
\geq 1\), so both \eqref{eq:da_min_form} and \eqref{eq:ds-omd_min_form_2} are identical. If the learning rate is constant, then all three formulas are equivalent. However, if the learning rate
is not constant, OMD \emph{is not} equivalent to the latter methods: changing the ordering of the subgradients \(\ghat_1, \dotsc, \ghat_t\) affects the choice of \(x_{t+1}\) in~\eqref{eq:omd_min_form}, while this does not happen in the other two formulas.

When \(\X\) is an arbitrary convex set, DA and DS-OMD are not
necessarily equivalent anymore due to the vectors $p_i'$ in the normal cone
of \(\X\). However, if we know that the iterates live in the relative interior of
\(\X\), the next lemma
shows that these vectors do not affect the set of minimizers
in~\eqref{eq:ds-omd_min_form_2}.

\begin{lemma}
    \label{lemma:normal_cone_relint}
    For any \(\xcirc \in \relint \X\) we have \(N_{\X}(\xcirc) =
    -N_{\X}(\xcirc)\). In particular, for any \(p
    \in N_{\X}(\xcirc)\) we have \(\iprod{p}{x} = \iprod{p}{\xcirc}\) for
    every \(x \in \X\).
\end{lemma}
\begin{proof}
    Let \(\xcirc \in \relint \X\). For the sake of contradiction, suppose there is \(p \in N_{\cX}(\xcirc)\) and \(x \in \X\) such that \(\iprod{p}{x - \xcirc} < 0\), that is, \(-p \not\in N_{\X}(\xcirc)\). Since \(\xcirc\) is in the relative interior of \(\X\), there is \(\mu > 1\) such that \(x_{\mu} \coloneqq \mu \xcirc + (1 - \mu)x \in \X \) (see~\citealp[Thm.~6.4]{roc70}). Then, 
    \begin{equation*}
        \iprod{p}{x_{\mu} - \xcirc} = 
        (1 - \mu) \iprod{p}{x - \xcirc} > 0,
    \end{equation*}
    a contradiction since \(p \in N_{\X}(\xcirc)\).
\end{proof}

With this lemma, we can easily derive simple and intuitive conditions
under which DS-OMD and DA are equivalent.
\comment{
By Lemma~\ref{lemma:projection} shows that the Bregman Projection
w.r.t.\ \(\Phi\) onto \(\X\) must be realized in \(\X \cap \Dcal\).
If \(\Dcal \cap \X \subseteq \interior \X\), then the iterates from
DS-OMD all lie in \(\interior \X\), where the normal cone of \(\X\)
is~\(\{0\}\) and the algorithms are equivalent. The next corollary
extends this observation yo the more general case where \(\X\) has empty
interior.}
\begin{corollary}
    \CorollaryName{EquivDSOMDAndDA}
    Let \(\Dcal \subseteq \Reals^n\) be the interior of the domain of
    \(\Phi\), let $\{x_t\}_{t \geq 1}$ be the DS-OMD iterates as
    in~\Algorithm{somd} and let $\{x_t'\}_{t \geq 1}$ be the DA
    iterates as in Algorithm~\ref{alg:aomd} with DA updates. If \(\Dcal
    \cap \X \subseteq \relint \X\) and \(x_1 = x_1'\), then \(x_t = x_t'\) for each \(t
    > 1\).
\end{corollary}
\begin{proof}
    Let \(t > 1\). Since \(x_t = \Pi_{\X}^{\Phi}(y_t)\), where
    \(y_t\) is as in Algorithm~\ref{alg:somd},
    Lemma~\ref{lemma:projection} implies~\(x_t \in \Dcal \cap \X
    \subseteq \relint \X\). By Lemma~\ref{lemma:normal_cone_relint} we
    have that the vectors on the normal cone
    in~\eqref{eq:ds-omd_min_form_2} do not affect the set of minimizers,
    which implies that~\eqref{eq:da_min_form}
    and~\eqref{eq:ds-omd_min_form_2} are equivalent.
\end{proof}
\vskip-1pt
An important special case of the above corollary is the prediction with
expert advice setting as in \Section{rwm-ds-omd}, where \(\Dcal=\Reals^n_{>0}\)
and $\X$ is the simplex $\Delta_n$. In this setting, \(\X \cap \Dcal
= \setst{x \in (0,1)^d}{\sum_{i =1}^n x_i = 1} = \relint \X\). By the
previous corollary, DS-OMD and DA produce the same iterates in this case,
even for dynamic learning rates. Classical OMD and DA were already known
to be equivalent in the experts' setting for a \emph{fixed} learning
rate~\citep[\S 5.4.2]{Hazan16}.
In contrast, with a dynamic learning rate, the DA and OMD iterates are certainly different, since OMD with a dynamic learning
rate may have linear regret~\citep{OrabonaP18}, whereas DA has sublinear regret.

\section{Dual-Stabilized OMD for Composite Functions}
\SectionName{regularized_problems}

In this section we extend the dual-stabilized OMD to the case where the
functions revealed at each round are
composite~\citep{Xiao,Duchi10compositeobjective}. More specifically, at each
round \(t \geq 1\) we see a function of the form \(f_t + \Psi\), where \(f_t\) and \(\Psi\) are convex functions, but the latter is
fixed and assumed to be ``easy'', i.e., we suppose we know
how to efficiently compute points in \(\argmin_{x \in \X}(D_{\Phi}(x,\xbar) +
\Psi(x))\). We could simply use the original dual-stabilized OMD in
this setting, but this approach has some drawbacks. One issue is that
subgradients of \(\Psi\) would end-up appearing in the regret bound from
Theorem~\ref{thm:somd}, which is not ideal: we want bounds that are unaffected
by the ``easy'' function \(\Psi\). Another drawback is that we would not be
using the knowledge of the structure of the functions, which may result in sub-optimal performance. For example, one may take \(\Psi = \norm{\cdot}_1\) hoping for sparse iterates. Yet, blindly applying OMD (and, thus, linearizing \(\Psi\)) does not yield sparse iterates~\citep{McMahan11a}. Finally, the
analysis of dual-stabilized OMD adapted to the composite setting is an easy extension of the original analysis of
\Section{StabilizedOMD}. Usually, algorithms for the composite setting require a more intricate analysis,
such as in the case of Regularized DA from~\citet{Xiao}, or the use of
powerful results, such as the duality between strong convexity and strong
smoothness used by~\citet{McMahan17}.
\cite{Duchi10compositeobjective} give regret bound a whose analysis is somewhat simpler and perhaps resembles ours.
Still, the techniques used in the latter do not directly apply when we use
dual-stabilization.

In the composite setting we assume without loss of generality that \(\X =
\Reals^n\) since we may substitute \(\Psi\) by \(\Psi + \indic[\X]{\cdot}\)
where \(\indic[\X]{x} = 0\) if \(x \in \X\) and is \(+\infty\) anywhere else. The \textbf{(effective) domain} of \(\Psi\) is the set \(\dom \Psi \subseteq \bR^n\) of points where \(\Psi\) is finite. To avoid confusion and make the effect of \(\Psi\) explicit, we
define the \textbf{\(\Psi\)-regret} of a sequence of functions \(\{f_t\}_{t \geq 1}\) and iterates \(\{x_t\}_{t \geq 1}\) (against a comparison point \(z \in \dom \Psi\)) by
\begin{equation*}
    \Reg^{\Psi}(T, z)
    \coloneqq 
    \sum_{t = 1}^T\big(f_t(x_t) + \Psi(x_t)\big) -
    \sum_{t = 1}^T\big( f_t(z) + \Psi(z)\big),
    \quad \forall T \geq 0.
\end{equation*}

To adapt the dual stabilization method to this setting, we use the same idea as
in~\cite{Duchi10compositeobjective}. Namely, we modify the proximal-like
formulation of dual stabilization from Proposition~\ref{prop:prox-ds-omd} so
that we do not linearize (i.e., take the subgradient) of  the function \(\Psi\).
This results in the following definition of dual-stabilized OMD for composite functions (given some \(x_1 \in \dom \Psi \cap \cD\)):
\begin{equation}
    \label{eq:def_reg_dsomd}
    \{x_{t+1}\} \coloneqq \argmin_{x \in \X}
    \paren[\Big]{ \gamma_t\paren[\big]{\eta_t(\iprod{\ghat_t}{x} + \Psi(x))
    + D_{\Phi}(x, x_t)} + (1 - \gamma_t) D_{\Phi}(x, x_1)},
    \quad \forall t \geq 1.
\end{equation}

\begin{algorithm}[t]
       \caption{Dual-stabilized OMD with dynamic learning
       rate $\eta_t$ and additional regularization function~\(\Psi\).}
       \label{alg:reg-somd}
       \begin{algorithmic}
         \STATE {\bfseries Input:}
         ${x}_1 \in \dom \Psi \cap \cD,\  
         \eta:\mathbb{N} \rightarrow \R_+,\
         \gamma : \mathbb{N} \to [0, 1] $
         \STATE $\hat{y}_1 = \nabla \Phi(x_1)$
         \FOR{$t=1,2,\ldots$}
            \STATE Incur cost $f_t(x_t)$ and receive $\ghat_t \in \partial f_t(x_t)$
            \STATE \vspace{-21pt} \addtolength{\jot}{-3pt}
                \begin{flalign}
                    & \hat{x}_t = \nabla \Phi(x_t) 
                    &&\text{$\rhd$ map primal iterate to dual space} 
                    \hspace{3cm}
                    \nonumber
                    \\
                    & \hat{w}_{t+1} = \hat{x}_t - \eta_t \ghat_t &&\text{$\rhd$ gradient step in dual space} \EquationName{CompGradientStep}
                    \\
                    & \hat{y}_{t+1} = \gamma_t \hat{w}_{t+1} + (1-\gamma_t) \hat{x}_1 &&\text{$\rhd$ stabilization in dual space} \EquationName{Comp-DSOMDHyp}
                    \\
                    & y_{t+1} = \nabla \Phi^*( \hat{y}_{t+1} ) &&\text{$\rhd$ map dual iterate to primal space}\nonumber
                    \\
                    & \alpha_{t+1} = \eta_t\gamma_t &&\text{$\rhd$ compute scaling factor for \(\Psi\)} \nonumber
                    \\
                    & x_{t+1} =
                \Pi^{\Phi}_{\alpha_{t+1}\Psi}(y_{t+1}) &&\text{$\rhd$ project onto feasible region} \EquationName{Comp-DSOMDHyp2}
                \end{flalign}
            \vspace{-21pt}
       \ENDFOR
    \end{algorithmic}
    \end{algorithm}

This equation defines the algorithm in a proximal form.
Due to the existence of $\Psi$, it is perhaps not obvious that it can also be written in pseudocode resembling \Algorithm{somd}.
Nevertheless, it can --- Algorithm~\ref{alg:reg-somd} presents
pseudocode equivalent to \eqref{eq:def_reg_dsomd}.
Interestingly, $\Psi$ appears in the pseudocode only in the projection step. In this new algorithm, we extend the
definition of Bregman Projection and define the \(\Psi\)-Bregman projection by
\(\{\Pi^{\Phi}_{\Psi}(y)\} \coloneqq \argmin_{x \in \Reals^n}(D_\Phi(x,y) +
\Psi(y))\).
\onlyJournal{For the sake of conciseness, we defer the proof of equivalence between Algorithm~\ref{alg:reg-somd} and \eqref{eq:def_reg_dsomd} to our technical report, but it boils down to properly interpreting the optimality conditions of~\eqref{eq:def_reg_dsomd}.}\onlyTR{
Let us first show that Algorithm~\ref{alg:reg-somd} produces the same iterates 
as~\eqref{eq:def_reg_dsomd}.
\begin{proposition}
    \label{prop:equiv}
    Let $\eta_t \geq 0$ and let $f_t \colon \dom \Psi \rightarrow \R$ be a convex function for each \(t \geq
 1\). Finally, let $\{x_t\}_{t \geq 1}$ be as in
 Algorithm~\ref{alg:reg-somd}. Then \(x_{t+1}\) satisfies~\eqref{eq:def_reg_dsomd} for each \(t \geq 1\).
\end{proposition}
\begin{proof}
    Let \(t \geq 1\).  By the optimality conditions of \(\Psi\)-Bregman projection (see Lemma~\ref{lemma:opt_condition}), we have
    \begin{equation}
        \label{eq:equiv_reg_1}
        \yhat_{t+1}
        \in \alpha_{t+1} \partial \Psi(x_{t+1})
        = \gamma_t \eta_t \partial \Psi(x_{t+1}).
    \end{equation}
    Moreover, by the definitions in Algorithm~\ref{alg:reg-somd} we have
    \begin{equation*}
        \yhat_{t+1} - \xhat_{t+1}
        = \gamma_t \what_{t+1} + (1 - \gamma_t) \xhat_1 - \xhat_{t+1}
        = \gamma_t (\xhat_t - \eta_t \ghat_t - \xhat_{t+1}) + (1 - \gamma_t) (\xhat_1 - \xhat_{t+1}).
    \end{equation*}
    By plugging the above into~\eqref{eq:equiv_reg_1} and rearranging, we have
    \begin{equation*}
        0 \in \gamma_t[\eta_t (\ghat_t + \partial \Psi(x_{t+1})) + \xhat_{t+1} - \xhat_{t}] + (1 - \gamma_t)(\xhat_{t+1} - \xhat_1).
    \end{equation*}
    Note that, since \(\nabla (D_{\Phi}(\cdot,b))(a) = \hat{a} - \hat{b}\) for any \(a,b \in \cD\), the right-hand side of the above inclusion is the subdifferential at \(x_{t+1}\) of the function being minimized in \eqref{eq:def_reg_dsomd}. By using Lemma~\ref{lemma:opt_condition} again, we conclude that \(x_{t+1}\) satisfies \eqref{eq:def_reg_dsomd}.
\end{proof}

} The next lemma shows an analogue of the generalized Pythagorean
theorem for the \(\Psi\)-Bregman Projection.

\begin{lemma}
    \label{lemma:regularized_bregman_pythagorean}
    Let \(\alpha > 0\) and \(\ybar \coloneqq \Pi_{\alpha \Psi}^{\Phi}(y)\). Then, 
\begin{equation*}
    D_{\Phi}(x, \ybar) + D_{\Phi}(\ybar, y)
    \leq D_{\Phi}(x,y) + \alpha(\Psi(x) - \Psi(\ybar)), \qquad
    \forall x \in \Reals^n.
\end{equation*}    
\end{lemma}
\begin{proof}
    By the optimality conditions for convex optimization, we have
    \begin{equation*}
        \nabla \Phi (y) -
    \nabla \Phi(\ybar) \in \subdiff[(\alpha\Psi)](\ybar).
    \end{equation*}
    Using the generalized
    triangle inequality for Bregman Divergences (Proposition~\ref{prop:BregmanTriangle}) and the subgradient inequality,
    we get
  \begin{equation*}
    D_\Phi(x, \ybar) + D_\Phi(\ybar, y) - D_\Phi(x,y)
    = \iprod{\nabla \Phi(y) - \nabla \Phi(\ybar)}{x - \ybar}
    \overset{\rm(i)}{\leq} \alpha(\Psi(x) - \Psi(\ybar)).
  \end{equation*}
  where (i) follows from \(\nabla \Phi (y) - \nabla \Phi(\ybar) \in \subdiff[(\alpha \Psi)](\ybar)\) and the convexity of $\alpha \Psi(\cdot)$.
\end{proof}

Finally, a regret bound such as the one we have for dual-stabilized OMD also holds in this setting when
using~Algorithm~\ref{alg:reg-somd}.

\begin{theorem} \TheoremName{RegularizedRegret}
Assume we have $\eta_t \geq \eta_{t+1} > 0$ for each~$t > 1$. Define $\gamma_t = \eta_{t+1}/\eta_t \in
(0,1]$ for all $t \geq 1$.  Let $\{f_t\}_{t \geq 1}$ be a sequence of
convex functions with $f_t \colon \bR^n \rightarrow \R$ for each \(t \geq
1\) and let \(\Phi \colon \bR^n \to \bR \cup \{+\infty\}\) be convex. Let $\{x_t\}_{t \geq 1}$ and $\{\what_t\}_{t \geq 2}$ be as in
Algorithm~\ref{alg:reg-somd}. Then, for all \(T > 0\) and \(z \in \dom \Psi\),
    \begin{equation}
        \label{eq:reg_somd_bound}
        \Reg^{\Psi}(T,z) ~\leq~ \sum_{t=1}^T \frac{\TriBreg{x_t}{x_{t+1}}{w_{t+1}}}{\eta_t} + \frac{ D_\Phi(z, x_1) }{\eta_{T+1}}, 
  \quad \forall T > 0.
    \end{equation}
\end{theorem}
\TRJournalSplit{
\begin{proof}[of \Theorem{RegularizedRegret}]
    Let \(z \in \dom \Psi\) and \(t \in \Naturals\).
    \begin{mdframed}
        By~\eqref{eq:regret_at_each_round}, we have
    \begin{equation}
        \label{eq:one_round_bound_regularized}
        f_t(x_t) + \Psi(x_t) - f_t(z) - \Psi(z)
        \leq 
        \begin{aligned}[t]
            \frac{1}{\eta_t} &\Big( D_\Phi(x_t, w_{t+1} ) + D_\Phi(z, x_t)\\ &- D_\Phi\big(z, {w}_{t+1}  \big)   \Big)
        + \Psi(x_t) - \Psi(z).
        \end{aligned}
    \end{equation}
    \end{mdframed}
    As in Theorem~\ref{thm:somd}, let us prove bound the above expression by
    something with telescoping terms.

\begin{claim}
    \ClaimName{CompositeDSOMD}
    Assume that $\gamma_t = \eta_{t+1}/\eta_t \in (0,1]$.
    Then
    \begin{alignat*}{3}
        \eqref{eq:one_round_bound_regularized}
        ~\leq~
        \frac{\TriBreg{x_t}{x_{t+1}}{w_{t+1}}}{\eta_t} 
            &+ \underbrace{\Big( \frac{1}{\eta_{t+1}} - \frac{1}{\eta_t} \Big)}_{\text{telescopes}} D_\Phi(z, x_1)
        \\
        & +
        \underbrace{\frac{D_\Phi(z, x_t)}{\eta_t} -\frac{D_\Phi(z, x_{t+1})}{\eta_{t+1}}}_{\text{telescopes}} 
        + \underbrace{\Psi(x_t) - \Psi(x_{t+1})}_{\text{telescopes}}.   
    \end{alignat*}
\end{claim}
\begin{proof}
    Fix \(z \in \dom \Psi \). First we derive the inequality
    \begin{alignat*}{3}
    &\gamma_t  \big( D_\Phi(z,w_{t+1}) - D_\Phi(x_{t+1},w_{t+1}) \big)
        \:+\: (1-\gamma_t)  &&D_\Phi(z,x_1)
        \\
    ~\geq~&     \gamma_t  \TriBreg{z}{x_{t+1}}{w_{t+1}}
        \:+\: (1-\gamma_t) \TriBreg{z}{x_{t+1}}{x_1}
        \qquad&&\text{(since $D_\Phi(x_{t+1},x_1) 
        \geq 0$ and $\gamma_t \leq 1$)}
        \\
    ~=~&        \TriBreg{z}{x_{t+1}}{y_{t+1}}
        \qquad &&\text{(by \Proposition{BregmanIdentity} and \eqref{eq:Comp-DSOMDHyp})}
        \\
    ~\geq~&     D_\Phi(z,x_{t+1})  + \alpha_{t+1} \big( \Psi( x_{t+1} ) - \Psi( z ) \big)
    \qquad &&\text{(by Lemma~\ref{lemma:regularized_bregman_pythagorean} and \eqref{eq:Comp-DSOMDHyp2})}.
\end{alignat*}
Rearranging and using both $\gamma_t>0$ and \(\alpha_{t+1} = \eta_t \gamma_t\)
yields
\begin{alignat}{3}
D_\Phi(z, w_{t+1})
    ~\geq~ 
        &D_\Phi(x_{t+1}, w_{t+1}) 
        - \Big(\frac{1}{\gamma_t}-1\Big) D_\Phi(z, x_1)
        \nonumber
        \\
        \nonumber
        \EquationName{CompDSOMDRearranged}
        &+ \frac{1}{\gamma_t} D_\Phi( z, x_{t+1} )
        + \eta_t \big( \Psi( x_{t+1} ) 
        - \Psi( z ) \big).
\end{alignat}
Plugging this into \eqref{eq:one_round_bound_regularized} 
yields
\begin{align*}
\eqref{eq:one_round_bound_regularized}
    ~=~& \frac{1}{\eta_t} \Big(
            D_\Phi(x_t,w_{t+1}) - D_\Phi(z,w_{t+1}) + D_\Phi(z,x_t)
        \Big) + \Psi(x_t) - \Psi(z)\\
    ~\leq~& 
    \begin{aligned}[t]
        \frac{1}{\eta_t} \Bigg(
        \TriBreg{x_t}{x_{t+1}}{w_{t+1}}
            + &\Big(\frac{1}{\gamma_t}-1\Big) D_\Phi(z, x_1)
            - \frac{1}{\gamma_t} D_\Phi( z, x_{t+1} )\\
            &+ D_\Phi(z,x_t)
        \Bigg) + \Psi(x_t) - \Psi(x_{t+1}).
    \end{aligned}
\end{align*}
The claim follows by the definition of $\gamma_t$. 
\end{proof}\begin{mdframed}
    The final step is very similar to the standard OMD proof. 
    Summing \eqref{eq:one_round_bound_regularized} over $t$
    and using \Claim{CompositeDSOMD} leads to the desired telescoping sum.
    \begin{alignat*}{3}
        &\sum_{t=1}^T \big( f_t(x_t) - f_t(z) \big)
        \\
            ~\leq~& \sum_{t=1}^T \Bigg(\frac{
            \TriBreg{x_t}{x_{t+1}}{w_{t+1}}}{\eta_t} &\:+\:& \Big( \frac{1}{\eta_{t+1}} - \frac{1}{\eta_t} \Big) D_\Phi(z, x_1)
            \nonumber
            \\
            &&&+\: \frac{D_\Phi(z, x_t)}{\eta_t} -\frac{D_\Phi(z, x_{t+1})}{\eta_{t+1}}
            +  \Psi(x_t) - \Psi(x_{t+1})\Bigg) 
            \\
            ~\leq~& \sum_{t=1}^T \frac{\TriBreg{x_t}{x_{t+1}}{w_{t+1}}}{\eta_t}
                &+&~~  \frac{ D_\Phi(z, x_1) }{\eta_{T+1}} + \Psi(x_1) - \Psi(x_{T+1})  \nonumber
            \\
            ~\leq~& \sum_{t=1}^T \frac{\TriBreg{x_t}{x_{t+1}}{w_{t+1}}}{\eta_t} &+&~~ \frac{ D_\Phi(z, x_1) }{\eta_{T+1}},
    \end{alignat*}
    where in the last step inequality we used \(x_1 \in \argmin_{x \in \X} \Psi(x)\). 
    \end{mdframed}
\end{proof}
}{
\begin{proof}\textbf{(Sketch)}
    The proof boils down to simple modifications to the original proof of Theorem~\ref{thm:somd}. We mostly have to replace \Claim{NewDSOMD} by the following claim, whose proof mimics the proof of the latter but replaces \Proposition{BregmanProjection} with Lemma~\ref   {lemma:regularized_bregman_pythagorean}.
    \begin{claim}
        \ClaimName{CompositeDSOMD}
        Assume that $\gamma_t = \eta_{t+1}/\eta_t \in (0,1]$.
        Then
        \begin{alignat*}{3}
            &\frac{1}{\eta_t}\Big( D_\Phi(x_t, w_{t+1} ) + &&D_\Phi(z, x_t) - D_\Phi\big(z, {w}_{t+1}  \big)   \Big)
        &&+ \Psi(x_t) - \Psi(z)\\
            \leq\;&
            \frac{\TriBreg{x_t}{x_{t+1}}{w_{t+1}}}{\eta_t} 
                &&+ \underbrace{\Big( \frac{1}{\eta_{t+1}} - \frac{1}{\eta_t} \Big)}_{\text{telescopes}} D_\Phi(z, x_1)
            && +
            \underbrace{\frac{D_\Phi(z, x_t)}{\eta_t} -\frac{D_\Phi(z, x_{t+1})}{\eta_{t+1}}}_{\text{telescopes}} 
            + \underbrace{\Psi(x_t) - \Psi(x_{t+1})}_{\text{telescopes}}.   
        \end{alignat*}
    \end{claim}
    Due to similarities with the proof of the previous regret bounds, we defer the complete proof to our technical report~\citep{TR}.
\end{proof}    
}

\acks{We thank Chris Liaw for pointing out a slight flaw in the proofs in an earlier draft of this paper.
We also thank Francesco Orabona for
suggesting the use of a slightly different definition of regret which
allows for more nuanced statements of our results.
We also express our
gratitude for the detailed feedback given by the three anonymous
reviewers from ICML 2020.}





\newpage

\appendix

\section{Standard facts}

\subsection{Scalar inequalities}


\begin{fact}
\FactName{quadratic}
For any $a > 0$ and $b, x \in \R$, we have $-a x^2 + b x \leq b^2 / 4a$.
\end{fact}

\begin{fact}
\FactName{ExpTaylor2}
$ e^{-x} \leq 1 - x + \frac{x^2}{2} $ for $x \geq 0$.
\end{fact}

\begin{fact}
\FactName{SquareRoot}
$\sum_{i=1}^t \frac{1}{\sqrt{i}} \leq 2\sqrt{t} - 1$ for $t \geq 1$.
\end{fact}

\begin{fact}
\FactName{LogBound}
$\log(x) \leq x-1$ for $x \geq 0$.
\end{fact}



The following proposition is a variant of an inequality
that is frequently used in online learning; see, e.g., \citet[Lemma 3.5]{AuerCG02}, \citet[Lemma 4]{McMahan17}. \onlyJournal{Since well-known proofs can be easily adapted to prove the version used here, we defer a complete proof to our technical report~\citep{TR}}.

\begin{restatable}{proposition}{Auer}
\label{prop:at_sum}
Let $u>0$ and $a_1, a_2, \ldots, a_T \in [0,u]$.
Then
\begin{equation*}
    \sum_{t=1}^T \frac{a_t}{\sqrt{ u + \sum_{i<t} a_i }}
~\leq~ 2 \sqrt{ \sum_{t=1}^T a_t }.
\end{equation*}
\end{restatable}
\onlyTR{
Although it is easy to prove this inequality by induction, the following proof may provide more intuition. 
The proof is based on a generic lemma on approximating sums by integrals.

\begin{lemma}[Sums with chain rule]
\label{lem:SumIntegral}
Let $S \subseteq \R$ be an interval.
Let $F : S \rightarrow \R$ be concave
and differentiable on the interior of $S$.
Let $u \geq 0$ and let $A : \{0,\ldots,T\} \rightarrow S$ satisfy $A(i)-A(i-1) \in [0,u]$ for each $1 \leq i \leq T$.
Then
$$
\sum_{i=1}^T F'\big( u+A(i-1) \big) \cdot (A(i) - A(i-1))
    ~\leq~ F(A(T))-F(A(0)).
$$
\end{lemma}

As $u \rightarrow 0$, the left-hand side becomes comparable to
$\int_0^T F'(A(x)) A'(x) \, dx$,
an expression that has no formal meaning since $A$ is only defined on integers.
If this expression existed, it would equal the right-hand side by
the chain rule.

\begin{proof}[of Lemma~\ref{lem:SumIntegral}]
Since $F$ is concave, $f:=F'$ is non-increasing.
Fix any $1 \leq i \leq T$
and observe that $f(x)  \geq f(A(i)) \geq f(u+A(i-1))$ for all $x \leq A(i)$.
Thus
\begin{equation*}
    f(u+A(i-1))\cdot(A(i)-A(i-1))
    ~\leq~ \int_{A(i-1)}^{A(i)} f(x) \, dx
    ~=~ F(A(i))-F(A(i-1)).
\end{equation*}
Summing over $i$, the right-hand side telescopes, which yields the result.
\end{proof}

\begin{proof}[of Proposition~\ref{prop:at_sum}]
Apply Lemma~\ref{lem:SumIntegral}
with $S = \R_{\geq 0}$, $F(x) = 2 \sqrt{x}$ and $A(i) = \sum_{1 \leq j \leq i} a_j$.
\end{proof}
}

We use following technical proposition in our proof of the first-order regret bounds from \Corollary{first-order-bound}. \onlyJournal{Since the proof is not enlightening nor complicated, we also defer it to our technical report~\citep{TR}.}

\begin{proposition}
\PropositionName{ineq_xy}
Let $x,y, \alpha, \beta > 0$.
\begin{align*}
\text{If}\quad
x-y &~\leq~ \alpha \sqrt{x} + \beta, \quad
\text{then}\quad
x-y ~\leq~ \alpha \sqrt{y} + \beta + \alpha \sqrt{\beta} + \alpha^2.
\end{align*}
\end{proposition}
\onlyTR{
\begin{proof}
The proposition's hypothesis yields
\[
y + \beta + \frac{\alpha^2}{4}
~\geq~ x - \alpha \sqrt{x} + \frac{\alpha^2}{4}
~=~ \left( \sqrt{x} - \frac{\alpha}{2} \right)^2.
\]
Taking the square root and rearranging,
\[
\sqrt{x} ~\leq~ 
    \sqrt{y + \beta + \frac{\alpha^2}{4}} + \frac{\alpha}{2}. 
\]
Squaring both sides and rearranging,
\begin{equation*}
x
~\leq~ y + \alpha \sqrt{y + \beta + \frac{\alpha^2}{4}} + \beta + \frac{\alpha^2}{2} ~\leq~ y + \alpha \sqrt{y} + \alpha\sqrt{\beta} + \beta + \alpha^2,
\end{equation*}
by subadditivity of the square root.
\end{proof}
}


\subsection{Bregman divergence properties}
\label{app:Bregman}

The following lemma collects basic facts regarding the a mirror map \(\Phi\) (or simply a function of Legendre type) and the Bregman
divergence it induces. See~\citet[Lemma~11.5 and Proposition~11.1]{game_prediction}.

\begin{lemma} \label{lem:bregman-div} The mirror map \(\Phi\) and the Bregman divergence it induces 
  satisfy the following properties:
\begin{itemize} 
    \item $D_\Phi(x, y)$ is convex in $x$.
    \item $\nabla \Phi ( \nabla \Phi^* (z) ) = z$ \ and \
      $\nabla^* \Phi ( \nabla \Phi (x) ) = x$ \ for all $x$ and $z$.
    \item
      $D_\Phi(x, y) = D_{ \Phi^* } ( \nabla \Phi(y), \nabla \Phi(x) )$
      \ for all $x$ and $y$.
\end{itemize}
\end{lemma}

\begin{proposition}
\PropositionName{BregmanStrong}
If $\Phi$ is $\rho$-strongly convex with respect to $\norm{\cdot}$ then
$D_\Phi(x,y) \geq \frac{\rho}{2} \norm{x-y}^2$.
\end{proposition}

\subsubsection{Differences of Bregman divergences}
\AppendixName{TriBreg}

Recall that in \eqref{eq:TriBreg} we defined the notation
\[
\TriBreg{a}{b}{c}
 ~\coloneqq~ D_\Phi(a,c) - D_\Phi(b,c)
 ~=~ \Phi(a)-\Phi(b) - \inner{ \nabla \Phi(c) }{ a-b }.
\]
This has several useful properties, which we now discuss.

\begin{proposition}
\PropositionName{TriBregLinear}
$\TriBreg{a}{b}{p}$ is linear in $\hat{p}$.
In particular,
\[
\TriBreg{a}{b}{\Grad \Phi^*(\hat{p}-\hat{q})}
    ~=~ \TriBreg{a}{b}{p} + \inner{\hat{q}}{a-b}
    \qquad\forall \hat{q} \in \bR^n.
\]
\end{proposition}
\begin{proof}
Immediate from the definition.
\end{proof}

\begin{proposition}
\PropositionName{GeneralBregmanTriangle}
For all $a,b,c,d \in \cD$,
$$
\TriBreg{a}{b}{d} - \TriBreg{a}{b}{c}
~=~
\inner{ \hat{c} - \hat{d} }{ a - b }
~=~ \TriBreg{a}{b}{d} + \TriBreg{b}{a}{c}.
$$
\end{proposition}
\begin{proof}
The first equality holds from \Proposition{TriBregLinear}
with $\hat{p} = \hat{c}$ and $\hat{q} = \hat{c} - \hat{d}$.
The second equality holds since $\TriBreg{b}{a}{c} = -\TriBreg{a}{b}{c}$.
\end{proof}

An immediate consequence is the ``generalized triangle inequality
for Bregman divergence'', such as in \citet[Eq.~(4.1)]{Bubeck15} or in \citet[Lemma 4.1]{BeckT03}.

\begin{proposition}
\PropositionName{BregmanTriangle}
For all $a,b,d \in \cD$,
$$ D_\Phi(a,d) - D_\Phi(b,d) + D_\Phi(b,a)
~=~
\inner{ \hat{a} - \hat{d} }{ a - b }
$$
\end{proposition}
\begin{proof}
Apply \Proposition{GeneralBregmanTriangle} with $c=a$ and use $D_\Phi(a,a)=0$.
\end{proof}

\begin{proposition}
\PropositionName{BregmanIdentity}
Let $a,b,c,u,v \in \R^n$ satisfy
$\gamma \hat{a} + (1-\gamma) \hat{b} = \hat{c}$
for some $\gamma \in \bR$.
Then
$$
        \gamma   \TriBreg{u}{v}{a}
\:+\: (1-\gamma) \TriBreg{u}{v}{b}
~=~              \TriBreg{u}{v}{c}.
$$
\end{proposition}

\begin{proof}
By definition of $D_\Phi$, the claimed identity is equivalent to
\begin{align*}
    \gamma \big( \Phi(u)-\Phi(v) - \inner{\Grad \Phi(a)}{u-v} \big)
    \:&+\:   (1-\gamma)   \big( \Phi(u)-\Phi(v) - \inner{\Grad \Phi(b)}{u-v} \big)
    \\
    &=~
                     \big( \Phi(u)-\Phi(v) - \inner{\Grad \Phi(c)}{u-v} \big).
\end{align*}
This equality holds by canceling $\Phi(u)-\Phi(v)$ and by the assumption that 
$\Grad \Phi(c) = (1-\gamma) \Grad\Phi(a) + \gamma \Grad\Phi(b)$.
\end{proof}

The following proposition is the ``Pythagorean theorem for Bregman divergence''.
Recall that $\Pi_\cX^\Phi(y) = \argmin_{u \in \cX} D_\Phi(u,y)$.
A proof may be found in \citet[Lemma 4.1]{Bubeck15}.

\begin{proposition}
\PropositionName{BregmanProjection}
Let $\cX \subset \bR^n$ be a convex set.
Let $p \in \bR^n$ and $\pi = \Pi_\cX^\Phi(p)$.
Then
$$
\TriBreg{z}{\pi}{p}
    ~\geq~ \TriBreg{z}{\pi}{\pi}
    ~=~    D_\Phi(z,\pi)
    \qquad\forall z \in \cX.
$$
\end{proposition}

A generalization of the previous proposition can be obtained by using linearity.

\begin{proposition}
\PropositionName{GeneralBregmanProjection}
Let $\cX \subset \bR^n$ be a convex set.
Let $p \in \bR^n$ and $\pi = \Pi_\cX^\Phi(p)$.
Then
$$
\TriBreg{v}{\pi}{\Grad \Phi^*(\hat{p}-\hat{q})}
    ~\geq~ \TriBreg{v}{\pi}{\Grad \Phi^*(\hat{\pi}-\hat{q})}
    \qquad\forall v \in \cX ,\, \hat{q} \in \bR^n.
$$
\end{proposition}
\begin{proof}
\begin{align*}
\TriBreg{v}{\pi}{\Grad \Phi^*(\hat{p}-\hat{q})}
    &~=~ \TriBreg{v}{\pi}{p} + \inner{\hat{q}}{v-\pi}
        &&\qquad\text{(by \Proposition{TriBregLinear})} \\
    &~\geq~ \TriBreg{v}{\pi}{\pi} + \inner{\hat{q}}{v-\pi}
        &&\qquad\text{(by \Proposition{BregmanProjection})} \\
    &~=~ \TriBreg{v}{\pi}{\Grad \Phi^*(\hat{\pi}-\hat{q})}
        &&\qquad\text{(by \Proposition{TriBregLinear})}.
\end{align*}
\end{proof}

\section{Additional proofs for \Section{rwm-ds-omd}}
\AppendixName{rwm-ds-omd}

An initial observation shows that
$\Lambda$ is non-negative in the experts' setting.

\begin{proposition}
\PropositionName{LambdaNonNeg}
$\Lambda(a,b) \geq 0$ for all $a \in \cX$, $b \in \cD$.
\end{proposition}

\begin{proof}
Let us write $\Lambda(a,b) = - \sum_{i=1}^n a_i \ln \frac{b_i}{a_i} +
\ln\big(\sum_{i=1}^n b_i\big)$. Since $a$ is a probability distribution,
we may apply Jensen's inequality to show that this expression is
non-negative.
\end{proof}


\begin{proofof}{\Proposition{dual_somd_KL}}
Since $a, b \in \cX$ we have $\norm{a}_1=\norm{b}_1=1$.
Then
\begin{alignat*}{3}
\TriBreg{a}{b}{c}
        &~=~ D_\KL(a,c) - D_\KL(b,c) \\
        &~=~ \big(D_\KL(a,c) + 1 - \norm{c}_1 + \ln\norm{c}_1\big)
           - \big(D_\KL(b,c) + 1 - &&\norm{c}_1 + \ln\norm{c}_1 \big) \\
        &~=~ \Lambda(a,c) - \Lambda(b,c)
        &&\qquad\text{(by definition of $\Lambda$)} \\
        &~\leq~ \Lambda(a,c)
        &&\qquad\text{(by \Proposition{LambdaNonNeg}).}
\end{alignat*}
\end{proofof}

\begin{proofof}{\Proposition{LambdaTaylor}}
Let $b = \Grad\Phi^*(\hat{a}-\eta\hat{q})$.
By \eqref{eq:ExpertMirror}, $b_i = a_i \exp(-\eta \hat{q}_i)$.
Then 
\begin{alignat*}{3}
\Lambda(a,\Grad\Phi^*(\hat{a}-\eta\hat{q}))
&~=~ \sum_{i=1}^n a_i \ln(a_i/b_i) + \ln \norm{b}_1 \\
&~=~ \sum_{i=1}^n \eta a_i \hat{q}_i + \ln\Big(\sum_{i=1}^n a_i \exp(-\eta \hat{q}_i) \Big)\\
&~\leq~ \sum_{i=1}^n \eta a_i \hat{q}_i + \sum_{i=1}^n a_i \exp(-\eta \hat{q}_i) - 1
&&\qquad\text{(by \Fact{LogBound})}\\
&~\leq~ \sum_{i=1}^n \eta a_i \hat{q}_i
 + \sum_{i=1}^n a_i \left( 1 - \eta \hat{q}_i + \frac{ \eta^2 \hat{q}_{i}^2 }{2} \right) - 1 &&\qquad\text{(by \Fact{ExpTaylor2})} \\
&~\leq~ \eta^2 \sum_{i=1}^n a_i \hat{q}_i/2,
\end{alignat*}
using $\sum_{i=1}^n a_i = 1$ (since $a \in \cX$) 
and $\hat{q}_i^2 \leq \hat{q}_i$ (since $\hat{q} \in [0,1]^n$).
\end{proofof}

\onlyTR{
\section{Remarks on lower bounds for the expert's problem}
\label{sec:app_lowerbound}

We have seen that dual averaging achieves
regret $\sqrt{ T \ln n }$ for all $T$.  Here we present a 
lower bound analysis for DA showing that this is the best one can hope for.

\begin{theorem}
\label{thm:lowerbound_arwm}
There exists a value of $n$ such that,
for every $T>0$, there exists a sequence of vectors $\{c_t \mid c_t \in \{0,1\}^n \}_{t=1}^T$ such that for cost functions \(f_t \coloneqq \iprod{c_t}{\cdot}\) for each \(t \in \{1, \cdots, T\}\) we have
\[
    \underset{t \to \infty}{\lim }\ \frac{ \mathrm{Regret}_{\text{DA} }(t)  }{ \sqrt{ t \ln n } }
    ~\geq~ 1,
\]
where $\mathrm{Regret}_{\text{DA} }(T)$ denotes the worst-case regret (that is, taking the supremum of the comparison point over the simplex) of the dual averaging algorithm used in \citep[Theorem 2.4]{introduction-online-optimization}.
\end{theorem}

It is known in the literature~\citep[\S 3.7]{game_prediction} that no algorithm can achieve a regret bound better than $\sqrt{\smash[b]{\sfrac{T}{2} \, \ln n}}$ for the problem of learning with expert advice (as $(T,n) \rightarrow \infty$).
Thus, there is still a $\sqrt{2}$ gap between the best upper and lower bounds (that hold for all $T$) for prediction with expert advice.
his gap was previously pointed out by \citet[pp.~52]{Ger11}. Even in a recent  study of lower-bounds on the regret of multiplicative-weights update methods by~\citet{GravinPS17a}, no lower-bounds better than \(\sqrt{\smash[b]{\sfrac{T}{2} \, \ln n}}\) for the anytime setting were known. In their work, they were interested in lower-bounds that were not asymptotic in \(n\) and studied many different families of algorithms parameterized by families of learning-rate schedules (fixed, decreasing with \(t\), deterministic, and randomized). Yet, they never restrict for learning rates that \emph{do not} depend on the time-horizon \(T\). Since MWU methods suffer at most \(\sqrt{\smash[b]{\sfrac{T}{2} \, \ln n}}\) regret for appropriately tuned learning rate (which depends on \(T\)), their approach should not and does not yield lower-bounds on the regret greater than \(\sqrt{\smash[b]{\sfrac{T}{2} \, \ln n}}\).

\begin{algorithm}[H]
    \caption{Adaptive randomized weighted majority based on DA~\citep[Theorem 2.4]{introduction-online-optimization}.}
    \label{alg:lazy_arwm_supp}
 \begin{algorithmic}
    \STATE {\bfseries Input:} $\eta:\mathbb{N} \rightarrow \R $
    \STATE $x_1 \coloneqq (1/n, 1/n, ...)^\intercal$
    \FOR{$t=1, 2, \ldots$}
         \STATE Incur cost $f_t(x_t) = \iprod{c_t}{x_t}$ and receive cost vector $c_t \in [0,1]^n.$
         \FOR{$j=1,2,\dots,n$}
             \STATE $y_{t+1,j} = x_{1,j} \exp \left( - \eta_t \sum_{k=1}^t c_{k}(j) \right)$
         \ENDFOR
         \STATE $x_{t+1} = y_{t+1} / \| y_{t+1} \|_1$
    \ENDFOR
 \end{algorithmic}
 \end{algorithm}

\begin{proof}
The detailed algorithm described by \citet[Theorem 2.4]{introduction-online-optimization} is shown in Algorithm~\ref{alg:lazy_arwm_supp}, where $\eta_t$ is set as $\sqrt{4\ln n/ t}$, we consider the case when $n=2$, and construct the following cost vectors:
\[
c_t =
\begin{cases}
(1, 0)^\intercal \qquad 1 \leq t <  \tau, t~\text{is odd} \\
(0, 1)^\intercal \qquad 1 \leq t <  \tau, t~\text{is even} \\
(1, 0)^\intercal \qquad \tau \leq t \leq T, 
\end{cases}
\qquad \forall t \geq 1,
\]
where $\tau \coloneqq \lfloor T - \log (T) \sqrt{T} \rfloor$.
Without loss of generality, we assume that $\tau$ is an odd number.

Throughout the remainder of this proof, denote \(\Regret(T)\) as the worts-case regret on \(T\) rounds, that is,
\begin{equation*}
    \Regret(T) \coloneqq \sup_{z \in \Delta_n} \Regret(T, z). 
\end{equation*}
It is obvious that the second expert is the best one and our regret at time $T$ is
\[
\mathrm{Regret}(T) = \sum_{ 1 \leq t < \tau } c_t^\intercal x_t -  \frac{\tau -1 }{2} + \sum_{ \tau \leq t \leq  T } c_t^\intercal x_t
\]
It is also easy to check that
\[
x_t =
\begin{cases}
(1/2, 1/2)^\intercal \qquad 1 \leq t <  \tau, t~\text{is odd} \\
\displaystyle \left(\frac{1}{1+\exp(\eta_{t-1}) } , \frac{1}{1 + \exp( -\eta_{t-1} )} \right)^\intercal \qquad 1 \leq t < \tau, t~\text{is even} \\
\displaystyle \left(\frac{1}{1+\exp(\eta_{t-1}(t - \tau) ) } , \frac{1}{1 + \exp( -\eta_{t-1}(t - \tau) )} \right)^\intercal \qquad \tau \leq t \leq T.
\end{cases}
\]
Thus,
\begin{align*}
    \mathrm{Regret}(T) = \underbrace{ \sum_{1 \leq t < \tau, t~\text{is even} } \left( \frac{1}{1 + \exp( -\eta_{t-1} )} - \frac{1}{2} \right) }_{\text{Term 1}}+ \underbrace{ \sum_{t=\tau}^T \frac{1}{1 + \exp((t - \tau)\eta_{t-1} )} }_{ \text{Term 2} }.
\end{align*}
Let us first look at Term 1. We have
\begin{align*}
    \text{Term 1} & \overset{\rm(i)} \geq  \sum_{1 \leq t < \tau, \; t~\text{is even} } \frac{1 - \exp( - \eta_{t-1} ) }{4} \\
     & = \sum_{1 \leq t < \tau, \; t~\text{is even} } \frac{\eta_{i-1}}{4} + O( \eta_{i-1}^2 ) \\
     & \overset{\rm(ii)} \geq \frac{ \sqrt{4 \ln n} }{4} \sqrt{\tau} + o(\tau)
\end{align*}
where (i) is true since $\frac{1}{2 - x} - \frac{1}{2} \geq \frac{x}{4}~$ for all $x \in (0, 1) $ and (ii) is true by using the fact that $\sum_{t=1}^\tau \frac{1}{\sqrt{t}} \geq 2\sqrt{\tau} - 2 $.

By the definition of $\tau$, we have $\displaystyle \lim_{T \to \infty} \frac{\tau}{T} = 1$, thus
\begin{align}
    \lim_{T \to \infty} \frac{ \text{Term 1} }{ \sqrt{T \ln n} } = \frac{1}{2}. \label{eq:term1}
\end{align}

For Term 2, we have
\begin{align}
    \text{Term 2} & =  \sum_{t=\tau}^T \frac{1}{1 + \exp( (t - \tau)\eta_{t-1} )} \nonumber \\
    & \geq \sum_{t=\tau}^T \frac{1}{1 + \exp((t - \tau)\sqrt{ \frac{4\ln n}{t-1} } )} \nonumber \\
    & \geq \sum_{t=\tau}^T \frac{1}{1 + \exp((t - \tau)\sqrt{ \frac{4\ln n}{\tau-1} } )} \nonumber \\
    & \geq \int_{t=\tau}^{T} \frac{1}{1 + \exp((t - \tau)\sqrt{ \frac{4\ln n}{\tau-1} } )} dt \nonumber \\
    & = \int_{y=0}^{ \log(T)\sqrt{T} } \frac{1}{1 + \exp(y\sqrt{ \frac{4\ln n}{\tau-1} } )} dy.  \label{eq:term2}
\end{align}
Note that 
\[
\int \frac{1}{1 + \exp(\beta y)} dy = y - \frac{\ln \left( 1 + e^{\beta y} \right)  }{\beta}.
\]
Set $\beta = \sqrt{ \frac{4\ln n}{\tau-1} }$ and plug the above result to Eq~\ref{eq:term2}, we get the following,
\begin{align*}
    \text{Term 2} \geq \log(T) \sqrt{T} - \frac{\ln\left(1 + \exp\left( \sqrt{ \frac{4\ln n}{\tau-1} } \log(T) \sqrt{T} \right) \right) }{ \sqrt{ \frac{4\ln n}{\tau-1} } } + \frac{\ln 2}{\sqrt{ \frac{4\ln n}{\tau-1} }}.
\end{align*}
Using the fact that $\ln \left( 1+e^x \right) = x + o(x)$,
\begin{align*}
    \text{Term 2} \geq \log(T) \sqrt{T} - \log(T) \sqrt{T} + o\left( \sqrt{ \frac{4\ln n}{\tau-1} } \log(T) \sqrt{T} \right) + \ln2 \sqrt{ \frac{  ( \tau - 1) }{4 \ln n} }.
\end{align*}
Note that $n = 2$, thus 
\begin{align}
    \lim_{T \to \infty} \frac{ \text{Term 2} }{ \sqrt{T \ln n} } = \lim_{T \to \infty} \frac{\ln2 \sqrt{ \frac{  ( \tau - 1) }{4 \ln 2} }}{ \sqrt{T \ln 2} } = \frac{1}{2}. \label{eq:term3}
\end{align}
Combining Eq.~\ref{eq:term1} and Eq.~\ref{eq:term3}, we conclude that
\[
    \lim_{T \to \infty} \frac{ \mathrm{Regret(T)} }{ \sqrt{T \ln n} } = \lim_{T \to \infty} \frac{ \text{Term 1} }{ \sqrt{T \ln n} } + \lim_{T \to \infty} \frac{ \text{Term 2} }{ \sqrt{T \ln n} } \geq  \frac{1}{2} + \frac{1}{2} = 1.
\]
\end{proof}
}

\section{Additional proofs for \Section{comparison}}
\AppendixName{comparison_app}

At many points throughout this section we will need to talk about optimality
condition for problems where we minimize a convex function over a convex set. Such conditions depend on the \emph{normal cone} of the set on which the
optimization is taking place.

\begin{lemma}[\protect{\citealp[Theorem~27.4]{roc70}}]
     \label{lemma:opt_condition}
    Let \(h \colon \Ccal \to \Reals\) be a closed convex function such that \((\relint \Ccal)\cap(\relint \X)\neq \emptyset \). Then, \(x \in \argmin_{z \in \X} h(z) \) if and only if there is \(\ghat \in \subdiff[h](x)\) such that \( -\ghat \in N_\X(x) \).
\end{lemma}

Using the above result allows us to derive a useful characterization of points that realize the Bregman projections. This result is similar to \citet[Lemma~4.1]{Bubeck15} and we defer the complete proof to our technical report~\citep{TR}.

\begin{lemma} \label{lemma:projection}
    Let \(y \in \Dcal\) and \(x \in \bar{\Dcal}\). Then \(x =
    \Pi_\X^\Phi(y)\) if and only if \(x \in \Dcal \cap \X \) and \(\nabla \Phi(y) - \nabla
    \Phi(x) \in N_\X(x)\).
\end{lemma}

\begin{proof}
    Suppose \(x \in \Dcal \cap \X \) and \(\nabla \Phi(y) - \nabla
    \Phi(x) \in N_\X(x)\). Since \(\nabla \Phi(y) - \nabla
    \Phi(x) = - \nabla(D_{\Phi}(\cdot,y))(x)\), by Lemma~\ref{lemma:opt_condition} we conclude that \(x \in \argmin_{z \in \X} D(z, y)\). Now suppose \(x =
    \Pi_\X^\Phi(y)\).  By Lemma~\ref{lemma:opt_condition} together with the definition of Bregman divergence, this is the case if and only if there is \(-g \in \subdiff[\Phi](x)\) such that \(-(g - \nabla \Phi(y)) \in N_\X(x)\). Since \(\Phi\) is of Legendre type we have \(\subdiff[\Phi](z) = \emptyset\) for any \(z \not\in \Dcal\) \citep[Theorem~26.1]{roc70}. Thus, \(x \in \Dcal\) and \(g = \nabla \Phi(x)\) since \(\Phi\) is differentiable. Finally, \(x \in \X\) by the definition of Bregman projection.
\end{proof}

For some of the proofs from \Section{comparison}, we need to one last result about the relation of subgradients and conjugate functions which is worth stating in full.

\begin{lemma}[\protect{\citealp[Theorem~23.5]{roc70}}]
    \label{lemma:subdiff_attainability}
    Let \(f \colon \X \to \Reals\), let \(x \in \X\) and let \(\yhat \in \Reals^n\). Then \( \yhat \in \subdiff(x)\) if and only if \(x\) attains \(\sup_{x \in
    \Reals^n}(\iprod{\yhat}{x} - f(x)) = f^*(\yhat)\).
\end{lemma}

\vskip 0.2in
\bibliography{ref}

\end{document}